\documentclass[twoside,11pt]{article}
\pdfoutput=1

\usepackage{geometry}
\usepackage{natbib}
\usepackage{graphicx}

\usepackage{amsmath}
\usepackage{amssymb}
\usepackage{theorem}
\newcommand{\BlackBox}{\rule{1.5ex}{1.5ex}}  % end of proof
\newenvironment{proof}{\par\noindent{\bf Proof\ }}{\hfill\BlackBox\\[2mm]}
 
\newtheorem{theorem}{Theorem}
 
\newtheorem{lemma}[theorem]{Lemma}

\newtheorem{remark}[theorem]{Remark}
\newtheorem{corollary}[theorem]{Corollary}

\newcommand{\eq}[1]{(\ref{#1})}
\newcommand{\mymatrix}[2]{\left[\begin{array}{#1} #2 \end{array}\right]}

\newcommand{\inner}[2]{\left\langle #1,#2 \right\rangle}
\newcommand{\rbr}[1]{\left(#1\right)}
\newcommand{\sbr}[1]{\left[#1\right]}
\newcommand{\cbr}[1]{\left\{#1\right\}}
\newcommand{\nbr}[1]{\left\|#1\right\|}
\newcommand{\abr}[1]{\left|#1\right|}

\newcommand{\RR}{\mathbb{R}}
\newcommand{\NN}{\mathbb{N}}

\newcommand{\one}{\mathbf{1}}

\newcommand{\Fcal}{\mathcal{F}}

\newcommand{\Xcal}{\mathcal{X}}

\newcommand{\Gcal}{\mathcal{G}}

\newcommand{\Ncal}{\mathcal{N}}

\newcommand{\Hcal}{\mathcal{H}}

\newcommand{\Eb}{\mathbf{E}}
\newcommand{\Pb}{\mathbf{P}}

\DeclareMathOperator*{\cov}{\mathrm{Cov}}
\DeclareMathOperator*{\var}{\mathrm{Var}}

\newcommand{\intset}[1]{\cbr{1..n}}

\usepackage[colorlinks,pdfpagelabels,plainpages=false]{hyperref}
\usepackage{algorithm}
\usepackage{algorithmic}
\usepackage{rotating}

\begin{document}

\title{Fastfood: Approximate Kernel Expansions in Loglinear Time}

\author{%
  Quoc Viet Le \\
  Google~Research, 1600 Amphitheatre Pky, Mountain View 94043 CA, USA
  \and
  Tamas Sarlos \\
  Google~Strategic Technologies, 1600 Amphitheatre Pky, Mountain View 94043 CA, USA
  \and
  Alexander J. Smola \\
  Carnegie Mellon University, 5000 Forbes Ave, Pittsburgh 15213
  PA, USA \\
  Google~Strategic Technologies, 1600 Amphitheatre Pky, Mountain View 94043 CA, USA
}

\maketitle

\begin{abstract}%
  \noindent 
  Despite their successes, what makes kernel methods difficult to use
  in many large scale problems is the fact that storing and computing
  the decision function is typically expensive, especially at
  prediction time. In this paper, we overcome this difficulty by
  proposing Fastfood, an approximation that accelerates such
  computation significantly. Key to Fastfood is the observation that
  Hadamard matrices, when combined with diagonal Gaussian matrices,
  exhibit properties similar to dense Gaussian random matrices. Yet
  unlike the latter, Hadamard and diagonal matrices are inexpensive to
  multiply and store. These two matrices can be used in lieu of
  Gaussian matrices in Random Kitchen Sinks proposed by
  \citet{RahRec09} and thereby speeding up the computation for a large
  range of kernel functions. Specifically, Fastfood requires $O(n \log
  d)$ time and $O(n)$ storage to compute $n$ non-linear basis
  functions in $d$ dimensions, a significant improvement from $O(n d)$
  computation and storage, without sacrificing accuracy.

  Our method applies to any translation invariant and any dot-product
  kernel, such as the popular RBF kernels and polynomial kernels. We
  prove that the approximation is unbiased and has low
  variance. Experiments show that we achieve similar accuracy to full
  kernel expansions and Random Kitchen Sinks while being 100x faster
  and using 1000x less memory. These improvements, especially in terms
  of memory usage, make kernel methods more practical for applications
  that have large training sets and/or require real-time prediction.
\end{abstract} 

%\keywords{Reproducing Kernel Hilbert Space, Random Kitchen Sinks,
%  Randomized Algorithms, Hadamard Transform}
\section{Introduction}

Kernel methods have proven to be a highly successful technique for
solving many problems in machine learning, ranging from classification
and regression to sequence annotation and feature extraction
\citep{BosGuyVap92,CorVap95,VapGolSmo97,TasGueKol04,SchSmoMul98}. At
their heart lies the idea that inner products in high-dimensional
feature spaces can be computed in implicit form via 
kernel function $k$:
\begin{align}
  k(x,x') = \inner{\phi(x)}{\phi(x')}.
\end{align}
Here $\phi: \Xcal \to \Fcal$ is a feature map transporting elements of
the observation space $\Xcal$ into a possibly infinite-dimensional
feature space $\Fcal$. This idea was first used by \cite{AizBraRoz64}
to show nonlinear separation. There exists a rich body of literature
on Reproducing Kernel Hilbert Spaces (RKHS)
\citep{Aronszajn44,Wahba90,Micchelli86b} and one may show that
estimators using norms in feature space as penalty are equivalent to
estimators using smoothness in an RKHS
\citep{Girosi98,SmoSchMul98}. Furthermore, one may provide a Bayesian
interpretation via Gaussian Processes. See
e.g. \citep{Williams98,Neal94,MacKay03} for details.

More concretely, to evaluate the decision function $f(x)$ on an example $x$,
one typically employs the kernel trick as follows
\begin{align}
  \nonumber 
  f(x) = \inner{w}{\phi(x)} 
  = \inner{\sum_{i=1}^N \alpha_i \phi(x_i)}{\phi(x)}
  = \sum_{i=1}^N \alpha_i k(x_i, x)
  \label{eq:fexpand}
\end{align}
This has been viewed as a strength of kernel methods, especially in
the days that datasets consisted of ten thousands of examples. This is
because the Representer Theorem of \cite{KimWah70} states that such a
function expansion in terms of finitely many coefficients must exist
under fairly benign conditions even whenever the space is infinite
dimensional. Hence we can effectively perform optimization in infinite
dimensional spaces.  This trick that was also exploited by
\cite{SchSmoMul98} for evaluating PCA. Frequently the coefficient
space is referred to as \emph{dual space}. This arises from the fact
that the coefficients are obtained by solving a dual optimization
problem.

Unfortunately, on large amounts of data, this expansion becomes a
significant liability for computational efficiency. For instance,
\cite{SteChr08} show that the number of nonzero $\alpha_i$ (i.e., $N$,
also known as the number of ``support vectors'') in many estimation
problems can grow linearly in the size of the training set. As a
consequence, as the dataset grows, the expense of evaluating $f$ also
grows. This property makes kernel methods expensive in many large
scale problems: there the sample size $m$ may well exceed billions of
instances. The large scale solvers of \cite{FanChaHsiWanetal08} and
\cite{MatVisSmo12} work in primal space to sidestep these problems,
albeit at the cost of limiting themselves to linear kernels, a
significantly less powerful function class. 

\section{Related Work}

Numerous methods have been proposed to mitigate this issue. To compare
computational cost of these methods we make the following assumptions:
\begin{itemize}
\item We have $m$ observations and access to an $O(m^\beta)$ with
  $\beta \geq 1$ algorithm for solving the optimization problem at
  hand. In other words, the algorithm is linear or worse. This is a
  reasonable assumption --- almost all data analysis algorithm need to
  inspect the data at least once to draw inference.
\item Data has $d$ dimensions. For simplicity we assume that it is
  dense with density rate $\rho$, i.e.\ on average $O(\rho d)$
  coordinates are nonzero.
\item The number of nontrivial basis functions is $O(\gamma m)$. This
  is well motivated by \cite{SteChr08} and it also follows from the
  fact that e.g.\ in regularized risk minimization the subgradient of
  the loss function determines the value of the associated dual
  variable.
\item We denote the number of (nonlinear) basis functions by $n$.
\end{itemize}

\paragraph{Reduced Set Expansions}

\cite{Burges96} focused on compressing function expansions after the
problem was solved by means of reduced-set expansions. That is, one
first solves the full optimization problem at $O(m^{\beta +1} \rho d)$
cost and subsequently one minimizes the discrepancy between the full
expansion and an expansion on a subset of basis functions. The
exponent of $m^{\beta + 1}$ arises from the fact that we need to
compute $O(m)$ kernels $O(m^\beta)$ times. Evaluation of the reduced
function set costs at least $O(n \rho d)$ operations per instance and
$O(n \rho d)$ storage, since each kernel function $k(x_i, \cdot)$
requires storage of $x_i$.

\paragraph{Low Rank Expansions}

Subsequent work by \cite{SmoSch00,FinSch01} and \cite{WilSee01} aimed to reduce
memory footprint and complexity by finding subspaces to expand
functions. The key difference is that these algorithms reduce the
function space \emph{before} seeing labels. While this is suboptimal,
experimental evidence shows that for well designed kernels the basis
functions extracted in this fashion are essentially as good as reduced
set expansions. This is to be expected. After all, the kernel encodes
our prior belief in which function space is most likely to capture the
relevant dependencies between covariates and labels. These
projection-based algorithms generate an $n$-dimensional subspace:
\begin{itemize}
\item Compute the kernel matrix $K_{nn}$ on an $n$-dimensional subspace at
  $O(n^2 \rho d)$ cost.
\item The matrix $K_{nn}$ is inverted at $O(n^3)$ cost.
\item For all observations one computes an explicit feature map by
  projecting data in RKHS onto the set of $n$ basis vectors via
  $\phi(x) = K_nn^{-\frac{1}{2}} \sbr{k(x_1, x), \ldots, k(x_n, x)}$. That is,
  training proceeds at $O(n \rho m^\beta + n^2 m)$ cost.
\item Prediction costs $O(n \rho d)$ computation and $O(n \rho d)$
  memory, as in reduced set methods, albeit with a different set of
  basis functions.
\end{itemize}
Note that these methods temporarily require $O(n^2)$ storage
during training, since we need to be able to multiply with the inverse
covariance matrix efficiently. This allows for solutions to problems
where $m$ is in the order of millions and $n$ is in the order of
thousands: for $n=10^4$ we need approximately $1$GB of memory to
store and invert the covariance matrix. Preprocessing can be
parallelized efficiently. Obtaining a minimal set of observations to
project on is even more difficult and only the recent work of
\cite{DasKem11} provides usable performance guarantees for it.

\paragraph{Multipole Methods}

Fast multipole expansions \citep{LeeGra09,GraMoo03b} offer one avenue
for efficient function expansions whenever the dimensionality of the
underlying space is relatively modest. However, for high dimensions
they become computationally intractable in terms of space
partitioning, due to the curse of dimensionality. Moreover, they are
typically tuned for localized basis functions, specifically the
Gaussian RBF kernel.

\paragraph{Random Subset Kernels}

A promising alternative to \emph{approximating an existing kernel
  function} is to design new ones that are immediately compatible with
scalable data analysis. A recent instance of such work is
the algorithm of \cite{DavGha14} who map observations $x$ into set
membership indicators $s_i(x)$, where $i$ denotes the random
partitioning chosen at iterate $i$ and $s \in \NN$ indicates the
particular set. 

While the paper suggests that the algorithm is scalable to large
amounts of data, it suffers from essentially the same problem as other
feature generation methods insofar as it needs to evaluate set
membership for each of the partitions for all data, hence we have an
$O(k n m)$ computational cost for $n$ partitions into $k$ sets on $m$
observations. Even this estimate is slightly optimistic since we
assume that computing the partitions is independent of the
dimensionality of the data. In summary, while the function class is
potentially promising, its computational cost considerably exceeds
that of the other algorithms discussed below, hence we do not
investigate it further.

\paragraph{Random Kitchen Sinks}

A promising alternative was proposed by \cite{RahRec09}
under the moniker of \emph{Random Kitchen Sinks}. In contrast to
previous work the authors attempt to obtain an \emph{explicit}
function space expansion directly. This works for translation
invariant kernel functions by performing the following operations:
\begin{itemize}
\item Generate a (Gaussian) random matrix $M$ of size $n \times d$.
\item For each observation $x$ compute $M x$ and apply a nonlinearity
  $\psi$ to each coordinate separately, i.e.\ $\phi_i(x) = \psi([M x]_i)$. 
\end{itemize}
The approach requires $O(n \times d)$ storage both at training and test
time. Training costs $O(m^\beta n \rho d)$ operations and prediction
on a new observation costs $O(n \rho d)$. 
This is potentially much cheaper than reduced set kernel
expansions. The experiments in \citep{RahRec09} showed that performance was
very competitive with conventional RBF kernel approaches while
providing dramatically simplified code. 

Note that explicit spectral finite-rank expansions offer potentially much
faster rates of convergence, since the spectrum decays as fast as the
eigenvalues of the associated regularization operator
\cite[]{WilSmoSch01}. Nonetheless Random Kitchen Sinks are a very
attractive alternative due to their simple construction and the
flexility in synthesizing kernels with predefined smoothness properties. 

\paragraph{Fastfood}

Our approach hews closely to random kitchen sinks. However, it
succeeds at overcoming their key obstacle --- the need to \emph{store} and
to \emph{multiply} by a random matrix. This way, fastfood, accelerates Random
Kitchen Sinks from $O(n d)$ to $O(n \log d)$ time while only requiring
$O(n)$ rather than $O(n d)$ storage. The speedup is most
significant for large input dimensions, a common case in many
large-scale applications. For instance, a tiny 32x32x3 image in
the CIFAR-10~\citep{Krizhevsky09} already has 3072 dimensions, and
non-linear function classes have shown to work well for
MNIST~\citep{SchSmo02} and CIFAR-10. Our approach relies on the fact
that Hadamard matrices, when combined with Gaussian scaling matrices,
behave very much like Gaussian random matrices. That means these two
matrices can be used in place of Gaussian matrices in Random Kitchen
Sinks and thereby speeding up the computation for a large range of
kernel functions. The computational gain is achieved because unlike
Gaussian random matrices, Hadamard matrices admit FFT-like
multiplication and require no storage. 

We prove that the Fastfood approximation is unbiased, has low
variance, and concentrates almost at the same rate as Random Kitchen
Sinks. Moreover, we extend the range of applications from radial basis
functions $k(\nbr{x-x'})$ to any kernel that can be written as dot
product $k(\inner{x}{x'})$. Extensive experiments with a wide range of
datasets show that Fastfood achieves similar accuracy to full kernel
expansions and Random Kitchen Sinks while being 100x faster with 1000x
less memory. These improvements, especially in terms of memory usage,
make it possible to use kernel methods even for embedded applications.

Our experiments also demonstrate that Fastfood, thanks to its speedup
in training, achieves state-of-the-art accuracy on the CIFAR-10
dataset~\citep{Krizhevsky09} among permutation-invariant methods.
Table~\ref{tab:compcost} summarizes the computational cost of the
above algorithms.

\begin{table}
  \centering
  \begin{tabular}{l|llll}
    Algorithm & CPU Training & RAM Training & CPU Test & RAM
    Test \\
    \hline
    Reduced set & $O(m^{\beta +1}\rho d + m n \rho d)$ & $O(\gamma m \rho
    d)$ & $O(n \rho d)$ & $O(n \rho d)$ \\
    Low rank & $O(m^\beta n \rho d + m n^2)$ & $O(n^2 + n \rho d)$ 
    & $O(n \rho d)$ & $O(n \rho d)$ \\
    Random Kitchen Sinks & $O(m^\beta n \rho d)$ & $O(n d)$ & $O(n
    \rho d)$ & $O(n d)$ \\
    Fastfood & $O(m^\beta n \log d)$ & $O(n)$ & $O(n \log d)$ & $O(n)$
  \end{tabular}
  \caption{Computational cost for reduced rank expansions. Efficient
    algorithms achieve $\beta = 1$ and typical sparsity coefficients
    are $\rho = 0.01$.}
  \label{tab:compcost}
\end{table}

Having an explicit function expansion is extremely beneficial from an
optimization point of view. Recent advances in both online
\citep{RatBagZin07} and batch \citep{TeoVisSmoLe10,BoyParChuPelEtal10}
subgradient algorithms summarily rely on the ability to compute
gradients in the feature space $\Fcal$ explicitly.

\section{Kernels and Regularization}

For concreteness and to allow for functional-analytic tools we need to
introduce some machinery from regularization theory and functional
analysis. The derivation is kept brief but we aim to be
self-contained. A detailed overview can be found e.g.\ in the books of
\cite{SchSmo02} and \cite{Wahba90}.

\subsection{Regularization Theory Basics}

When solving a regularized risk minimization problem one needs to
choose a penalty on the functions employed. This can be achieved e.g.\
via a simple norm penalty on the coefficients
\begin{align}
  f(x) = \inner{w}{\phi(x)}
  \text{with penalty }
  \Omega[w] = \frac{1}{2} \nbr{w}_2^2.
\end{align}
Alternatively we could impose a smoothness requirement which emphasizes
simple functions over more complex ones via
\begin{align}
  \nonumber
  \Omega[w] = \frac{1}{2} \nbr{f}_\Hcal^2
  \text{ such as }
  \Omega[w] = \frac{1}{2} \sbr{\nbr{f}_{L_2}^2 + \nbr{\nabla f}_{L_2}^2}
\end{align}
One may show that the choice of feature map $\phi(x)$ and RKHS norm
$\nbr{\cdot}_\Hcal$ are connected. This is formalized in the
reproducing property
\begin{align}
  f(x) = \inner{w}{\phi(x)}_\Fcal = \inner{f}{k(x, \cdot)}_\Hcal.
\end{align}
In other words, inner products in feature space $\Fcal$ can be viewed
as inner products in the RKHS. An immediate consequence of the above
is that $k(x,x') = \inner{k(x,\cdot)}{k(x',\cdot)}_\Hcal$. It also
means that whenever norms can be written via regularization operator $P$,
we may find $k$ as the Greens function of the operator. That is,
whenever $\nbr{f}_\Hcal^2 = \nbr{P f}^2_{L_2}$ we have
\begin{align}
  f(x) = \inner{f}{k(x,\cdot)}_\Hcal = \inner{f}{P^{\dag} P
    k(x,\cdot)} = \inner{f}{\delta_x}.
\end{align}
That is, $P^\dag P k(x,\cdot)$ as like a delta distribution on
$f \in \Hcal$. This allows us to identify $P^\dag P$ from $k$ and vice
versa \citep{SmoSchMul98,Girosi98,GirJonPog95,GirAnz93,Wahba90}. Note, though,
that this need not uniquely identify $P$, a property that we will be
taking advantage of when expressing a given kernel in terms of global
and local basis functions. For instance, any isometry $U$ with $U^\top
U = \one$ generates an equivalent $P' = U P$. In other words, there
need not be a unique feature space representation that generates a
given kernel (that said, all such representations are equivalent).

\subsection{Mercer's Theorem and Feature Spaces}
\label{sec:mercer}

A key tool is the theorem of \cite{Mercer09} which guarantees that
kernels can be expressed as an inner product in some Hilbert space.
\begin{theorem}[Mercer]
  Any kernel $k: \Xcal \times \Xcal \to \RR$ satisfying Mercer's
  condition
  \begin{align}
    \int k(x,x') f(x) f(x') dx dx' \geq 0 \text{ for all } f \in
    L_2(\Xcal)
  \end{align}
  can be expanded into
  \begin{align}
    \label{eq:mercer}
    k(x,x') = \sum_j \lambda_j \phi_j(x) \phi_j(x')
    \text{ with }
    \lambda_j \geq 0
    \text{ and }
    \inner{\phi_i}{\phi_j} = \delta_{ij}.
  \end{align}
\end{theorem}
The key idea of \cite{RahRec08} is to use sampling to
approximate the sum in \eq{eq:mercer}. Note that for trace-class
kernels, i.e.\ for kernels with finite $\sum_j \lambda_j$ we can
normalize the sum to mimic a probability distribution, i.e.\ we have
\begin{align}
  k(x,x') = \nbr{\lambda}_1 \Eb_{\lambda} \sbr{\phi_\lambda(x) \phi_\lambda(x')}
  \text{ where }
  p(\lambda) =
  \begin{cases}
    \nbr{\lambda}_1^{-1} {\lambda} & \text{ if } \lambda \in \cbr{
      \ldots \lambda_j \ldots} \\
    0 & \text{ otherwise}
  \end{cases}
\end{align}
Consequently the following approximation converges for $n \to \infty$
to the true kernel
\begin{align}
  \label{eq:sample-rks}
  \lambda_i \sim p(\lambda) 
  \text{ and } 
  k(x,x') \approx \frac{\nbr{\lambda}_1}{n} \sum_{i=1}^n
  \phi_{\lambda_i}(x) \phi_{\lambda_i}(x') 
\end{align}
Note that the basic connection between random basis functions was well
established, e.g.,\ by \citet{Neal94} in proving that the Gaussian
Process is a limit of an infinite number of basis functions.  A
related strategy can be found in the so-called `empirical' kernel map
\citep{TsuKinAsa02,SchSmo02} where kernels are computed via
\begin{align}
  k(x,x') = \frac{1}{n} \sum_{i=1}^n \kappa(x_i, x) \kappa(x_i, x')
\end{align}
for $x_i$ often drawn from the same distribution as the training
data. An explicit expression for this map is given e.g.\ in
\citep{SmoSchMul98b}. 
The expansion \eq{eq:sample-rks} is possible whenever the following
conditions hold:
\begin{enumerate}
\item An inner product expansion of the form \eq{eq:mercer} is known
  for a given kernel $k$.
\item \label{issue:fast} The basis functions $\phi_j$ are sufficiently inexpensive to
  compute.
\item The norm $\nbr{\lambda}_1$ exists, i.e.,\ $k$ corresponds to a
  trace class operator \cite{Kreyszig89}.
\end{enumerate}
Although condition~\ref{issue:fast} is typically difficult to achieve,
there exist special classes of expansions that are computationally
attractive.  Specifically, whenever the kernels are invariant under the
action of a symmetry group, we can use the eigenfunctions of its
representation to diagonalize the kernel.

\subsection{Kernels via Symmetry Groups}

Of particular interest in our case are kernels with some form of group
invariance since in these cases it is fairly straightforward to identify
the basis functions $\phi_i(x)$. The reason is
that whenever $k(x,x')$ is invariant under a symmetry group
transformation of its arguments, it means that we can find a matching
eigensystem efficiently, simply by appealing to the functions that
decompose according to the irreducible representation of the
group. 
\begin{theorem}
  Assume that a kernel $k: \Xcal^2 \to \RR$ is invariant under the action of a
  symmetry group $\Gcal$, i.e.\ assume that $k(x,x') = k(g \circ x, g
  \circ x')$ holds for all $g \in \Gcal$. In this case, the
  eigenfunctions $\phi_i$ of $k$ can be decomposed according to the irreducible
  representations of $\Gcal$ on $k(x, \cdot)$. The eigenvalues within
  each such representation are identical. 
\end{theorem}
For details see e.g.\ \cite{BerChrRes84}. This means that knowledge of
a group invariance dramatically simplifies the task of finding an
eigensystem that satisfies the Mercer decomposition. Moreover, by
construction unitary representations are orthonormal. 
\paragraph{Fourier Basis}
To make matters more concrete, consider translation invariant
kernels
\begin{align}
  k(x,x') = k(x-x',0).
\end{align}
The matching symmetry group is translation group with the Fourier
basis admitting a unitary irreducible representation. Corresponding
kernels can be expanded
\begin{align}
  k(x,x') = \int_z dz \exp\rbr{i \inner{z}{x}} \exp\rbr{-i
    \inner{z}{x'}} \lambda(z)
  = \int_z dz \exp\rbr{i \inner{z}{x-x'}} \lambda(z).
\end{align}
This expansion is particularly simple since the translation group
is Abelian.  By construction the function $\lambda(z)$ is obtained by
applying the Fourier transform to $k(x,0)$ --- in this case the above
expansion is simply the inverse Fourier transform. We have
\begin{align}
  \lambda(z) = (2\pi)^{-d} \int dx \exp\rbr{-i \inner{x}{z}} k(x,0).
\end{align}
This is a well studied problem and for many kernels we may obtain
explicit Fourier expansions. For instance, for Gaussians it is a
Gaussian with the inverse covariance structure. For the Laplace kernel
it yields the damped harmonic oscillator spectrum. That is, good
choices of $\lambda$ are
\begin{align}
  \label{eq:rbfspectrum}
  \lambda(z) & = (2\pi)^{-\frac{d}{2}} \sigma^{d} e^{-\frac{1}{2
      \sigma^2} \nbr{z}^2_2} &
  \text{(Gaussian RBF Kernel)} \\
  \lambda(z) & = \bigotimes_{j=1}^l 1_{U_d}(z) & 
  \text{(Matern Kernel)}
  \label{eq:besselspectrum}
\end{align}
Here the first follows from the fact that Fourier transforms of
Gaussians are Gaussians and the second equality follows from the fact
that the Fourier spectrum of Bessel functions can be expressed as
multiple convolution of the unit sphere. For instance, this includes
the Bernstein polynomials as special case for one-dimensional
problems. For a detailed discussion of spectral properties for a broad
range of kernels see e.g.\ \citep{Smola98}. 

\paragraph{Spherical Harmonics}

Kernels that are rotation invariant can be written as an expansion of
spherical harmonics. \citep[Theorem 5]{SmoOvaWil01} shows that
dot-product kernels of the form $k(x,x') = \kappa(\inner{x}{x'})$ can
be expanded in terms of spherical harmonics. This provides necessary
and sufficient conditions for certain families of kernels. Since
\cite{SmoOvaWil01} derive an incomplete characterization involving an
unspecified radial contribution we give a detailed derivation below.
\begin{theorem}
  \label{th:polyexpand}
  For a kernel $k(x,x') = \kappa(\inner{x}{x'})$ with $x,x' \in \RR^d$
  and with analytic $\kappa$, the eigenfunction expansion can be
  written as
\begin{align}
  \label{eq:yln}
  k(x,x') & = \sum_{n} \frac{\Omega_{d-1}}{N(d,n)} \lambda_{n} \nbr{x}^n \nbr{x'}^n \sum_j
  Y_{n,j}^d\rbr{\frac{x}{\nbr{x}}} Y_{n,j}^d\rbr{\frac{x'}{\nbr{x'}}} \\
  \label{eq:legendre}
  & = \sum_{n} \lambda_{n} \nbr{x}^n \nbr{x'}^n L_{n,d}
  \rbr{\frac{\inner{x}{x'}}{\nbr{x} \nbr{x'}}}  \\
  & = \sum_n \frac{N(d,n)}{\Omega_{d-1}} \lambda_n \nbr{x}^n \nbr{x'}^n \int_{S_d} 
  L_{n,d} \rbr{\nbr{x}^{-1} \inner{x}{z}}
  L_{n,d} \rbr{\nbr{x'}^{-1} \inner{x'}{z}} dz
  \label{eq:integral}
\end{align}
Here $Y_{n,j}^d$ are orthogonal polynomials of degree $n$ on the
$d$-dimensional sphere. Moreover, $N(d,n)$ denotes the number of
linearly independent homogeneous polynomials of degree $n$ in $d$
dimensions, and $\Omega_{d-1}$ denotes the volume of the $d-1$
dimensional unit ball. $L_{n,d}$ denotes the Legendre polynomial of
degree $n$ in $d$ dimensions. Finally, $\lambda_n$ denotes the
expansion coefficients of $\kappa$ in terms of $L_{n,d}$.
\end{theorem}
\begin{proof}
  Equality between the two expansions follows from the addition theorem of spherical
  harmonics of order $n$ in $d$ dimensions. Hence, we only need to
  show that for $\nbr{x} = \nbr{x'} = 1$ the expansion
  $\kappa(\inner{x}{x'}) = \sum_n \lambda_n L_{n,d}(\inner{x}{x'})$
  holds. 

  First, observe, that such an expansion is always possible since the
  Legendre polynomials are orthonormal with respect to the measure
  induced by the $d-1$ dimensional unit sphere, i.e.\ with respect to
  $(1-t^2)^{\frac{d-3}{2}}$. See e.g.\ \cite[Chapter 3]{Hochstadt61} for
  details. Hence they form a complete basis for one-dimensional
  expansions of $\kappa(\xi)$ in terms of $L_{n,d}(\xi)$. Since
  $\kappa$ is analytic, we can extend the homogeneous polynomials
  radially by expanding according to \eq{eq:legendre}. This proves the
  correctness.

  To show that this expansion provides necessary and sufficient
  conditions for positive semidefiniteness, note that $Y_{l,n}^d$ are
  orthogonal polynomials. Hence, if we had $\lambda_{n} < 0$ we could
  use any matching $Y_{l,n}^d$ to falsify the conditions of Mercer's
  theorem. 

  Finally, the last equality follows from the fact that
  $\int_{S_{d-1}} Y_{l,n}^d Y_{l',n}^d = \delta_{l,l'}$, i.e.\ the
  functions $Y_{l,n}^d$ are orthogonal polynomials. Moreover, we use
  the series expansion of $L_{n,d}$ that also established equality
  between the first and second line. 
\end{proof}
The integral representation of \eq{eq:integral} may appear to be
rather cumbersome. Quite counterintuitively, it holds the key to a
computationally efficient expansion for kernels depending on
$\inner{x}{x'}$ only. This is the case since we may sample from a
spherically isotropic distribution of unit vectors $z$ and compute
Legendre polynomials accordingly. As we will see, computing inner
products with spherically isotropic vectors can be accomplished very
efficiently using a construction described in
Section~\ref{sec:sampling-bf}. 
\begin{corollary}
  \label{cor:sample-poly}
  Denote by $\lambda_n$ the coefficients obtained by a Legendre
  polynomial series expansion of $\kappa(\inner{x}{x'})$ and let
  $N(d,n) = \frac{(d+n-1)!}{n!(d-1)!}$ be the number of linearly
  independent homogeneous polynomials of degree $n$ in $d$
  variables. Draw $z_i \sim S_{d-1}$ uniformly from the unit sphere
  and draw $n_i$ from a spectral distribution with $p(n) \propto
  \lambda_n N(d,n)$. Then 
  \begin{align}
    \Eb\sbr{m^{-1} \sum_{i=1}^m L_{n_i,d}(\inner{x}{z_i})
      L_{n_i,d}(\inner{x'}{z_i})} = \kappa(\inner{x}{x'})
  \end{align}
\end{corollary}
In other words, provided that we are able to compute the Legendre
polynomials $L_{n,d}$ efficiently, and provided that it is possible to
draw from the spectral distribution of $\lambda_n N(d,n)$, we have an
efficient means of computing dot-product kernels. 

For kernels on the symmetric group that are invariant under group
action, i.e.\ kernels satisfying $k(x,x') = k(g \circ x, g \circ x')$
for permutations, expansions using Young Tableaux can be found in
\citep{HuaGueGui07}. A very detailed discussion of kernels on
symmetry groups is given in \cite[Section 4]{Kondor08}. However,
efficient means of computing such kernels rapidly still remains an
open problem.

\subsection{Explicit Templates}

In some cases expanding into eigenfunctions of a symmetry group may be
undesirable. For instance, the Fourier basis is decidedly nonlocal and
function expansions using it may exhibit undesirable local deviations,
effectively empirical versions of the well-known Gibbs
phenomenon. That is, local changes in terms of observations can have
far-reaching global effects on observations quite distant from the
observed covariates. 

This makes it desirable to expand estimates in terms of localized
basis functions, such as Gaussians, Epanechikov kernels, B-splines or
Bessel functions. It turns out that the latter is just as easily
achievable as the more commonplace nonlocal basis function
expansions. Likewise, in some cases the eigenfunctions are expensive
to compute and it would be desirable to replace them with possibly
less statistically efficient alternatives that offer cheap
computation.  

Consequently we generalize the above derivation to general nonlinear
function classes dependent on matrix multiplication or distance
computation with respect to spherically symmetric sets of
instances. The key is that the feature map depends on $x$
only via
\begin{align}
  \label{eq:kappatemplate}
  \phi_z(x) := \kappa(x^\top z, \nbr{x}, \nbr{z}) \text{ for } x, z \in \RR^d
\end{align}
That is, the feature map depends on $x$ and $z$ only in terms of their
norms and an inner product between both terms. Here the dominant cost
of evaluating $\phi_z(x)$ is the inner product $x^\top z$. All other
operations are $O(1)$, provided that we computed $\nbr{x}$ and
$\nbr{z}$ previously as a one-off operation. Eq.~\eq{eq:kappatemplate}
includes the squared distance as a special case:
\begin{align}
  \nbr{x-z}^2 & = \nbr{x}^2 + \nbr{z}^2 - 2 x^\top z  \text{ and } \\
  \kappa(x^\top z, \nbr{x}, \nbr{z}) & := \kappa(\nbr{x-z}^2)
\end{align}
Here $\kappa$ is suitably normalized, such as $\int dz \kappa(z) =
1$. In other words, we expand $x$ in terms of how close the
observations are to a set of well-defined anchored basis functions. It
is clear that in this case
\begin{align}
  \label{eq:localized}
  k(x,x') & := \int d\mu(z) \kappa_z(x-z) \phi_z(x'-z) 
\end{align}
is a kernel function since it can be expressed as an inner
product. Moreover, provided that the basis functions $\phi_z(x)$ are
well bounded, we can use sampling from the (normalized) measure
$\mu(z)$ to obtain an approximate kernel expansion 
\begin{align}
  k(x,x') = \frac{1}{n} \sum_{i=1}^n \kappa(x-z_i) \kappa(x'-z_i).
\end{align}
Note that there is no need to obtain an explicit closed-form expansion
in \eq{eq:localized}. Instead, it suffices to show that this expansion
is well-enough approximated by draws from $\mu(z)$.

\paragraph{Gaussian RBF Expansion}

For concreteness consider the following:
\begin{align}
  \label{eq:gaussgauss}
  \phi_z(x) & = \exp\sbr{-\frac{a}{2} \sbr{\nbr{x}^2 - 2 x^\top z + \nbr{z}^2}} 
  \text{ and }
  \mu(z) := \exp\sbr{-\frac{b}{2} \nbr{z}^2}
\end{align}
Integrating out $z$ yields
\begin{align}
  k(x,x') \propto \exp\sbr{-\frac{a}{2} \frac{b}{2a+b}
    \sbr{\nbr{x}^2 + \nbr{x'}^2}
    -\frac{a^2}{4a+2b} \nbr{x-x'}^2}.
\end{align}
This is a locally weighted variant of the conventional Gaussian RBF
kernel, e.g.\ as described by \cite{Haussler99}. While
this loses its translation invariance, one can easily verify that for
$b \to 0$ it converges to the conventional kernel. Note that the key
operation in generating an explicit kernel expansion is to evaluate
$\nbr{z_i - x}$ for all $i$. We will explore settings where this can
be achieved for $n$ locations $z_i$ that are approximately random at
only $O(n \log d)$ cost, where $d$ is the dimensionality of the
data. Any subsequent scaling operation is $O(n)$, hence negligible in
terms of aggregate cost. Finally note that by dividing out the terms
related only to $\nbr{x}$ and $\nbr{x'}$ respectively we obtain a
'proper' Gaussian RBF kernel. That is, we use the following features:
\begin{align}
  \tilde\phi_z(x) & = \exp\sbr{-\frac{a}{2}\sbr{\frac{2a}{2a+b}
      \nbr{x}^2 - 2x^\top z + \nbr{z}^2}} \\
  \text{ and }
  \mu(z) & = \exp\sbr{-\frac{b}{2} \nbr{z}^2}.
\end{align}
Weighting functions $\mu(z)$ that are more spread-out than a Gaussian
will yield basis function expansions that are more adapted to
heavy-tailed distributions. It is easy to see that such expansions can
be obtained simply by specifying an \emph{algorithm} to draw $z$
rather than having to express the kernel $k$ in closed form at all. 

\paragraph{Polynomial Expansions}

One of the main inconveniences in computational evaluation of
Corollary~\ref{cor:sample-poly} is that we need to evaluate the
associated Legendre polynomials $L_{n,d}(\xi)$ directly. This is
costly since currently there are no known $O(1)$ expansions for the
\emph{associate} Legendre polynomials, although approximate $O(1)$ variants for
the regular Legendre polynomials exist \citep{BogMicFos12}. 
This problem can be alleviated by considering the following form of
polynomial kernels:
\begin{align}
  \label{eq:directpoly}
  k(x,x') = \sum_p \frac{c_p}{\abr{S_{d-1}}} \int_{S_{d-1}} \inner{x}{v}^p
  \inner{x'}{v}^p dv
\end{align}
In this case we only need the ability to draw from the uniform
distribution over the unit sphere to compute a kernel. The price to be paid for this is
that the effective basis function expansion is rather more complex. To
compute it we use the following tools from the theory of special functions.
\begin{itemize}
\item For fixed $d \in \NN_0$ the associated kernel is a homogeneous
  polynomial of degree $d$ in $x$ and $x'$ respectively and it only
  depends on $\nbr{x}, \nbr{x'}$ and the cosine of the angle $\theta :=
  \frac{\inner{x}{x'}}{\nbr{x} \nbr{x'}}$ between both vectors. This
  follows from the fact that convex combinations of homogeneous
  polynomials remain homogeneous polynomials. Moreover, the dependence
  on lenghts and $\theta$ follows from the fact that the expression is
  rotation invariant.
\item The following integral has a closed-form solution for $b \in
  \NN$ and for even $a$.
  \begin{align}
    \int_{-1}^{1} x^a (1-x^2)^{\frac{b}{2}} dx =
    \frac{\Gamma\rbr{\frac{a+1}{2}} \Gamma\rbr{\frac{b+3}{2}}}{\Gamma\rbr{\frac{a+b+3}{2}}}
  \end{align}
  For odd $a$ the integral vanishes, which follows immediately from
  the dependence on $x^a$ and the symmetric domain of integration
  $[-1, 1]$.
\item The integral over the unit-sphere $S_{d-1} \in \RR^d$ can be
  decomposed via
  \begin{align}
    \int_{S_{d-1}} f(x) dx = \int_{-1}^1 \sbr{\int_{S_{d-2}} f\rbr{x_1,
    x_2 \sqrt{1-x_1^2}}  dx_2 } (1-x_1)^{\frac{d-3}{2}} dx_1
  \end{align}
  That is, we decompose $x$ into its first coordinate $x_1$ and the
  remainder $x_2$ that lies on $S_{d-2}$ with suitable rescaling by
  $(1-x_1)^{\frac{d-3}{2}}$. Note the exponent of $\frac{d-3}{2}$ that
  arises from the curvature of the unit sphere. See e.g.\
  \cite[Chapter 6]{Hochstadt61} for details.
\end{itemize}
While \eq{eq:directpoly} offers a simple expansion for sampling, it is
not immediately useful in terms of describing the kernel as a function
of $\inner{x}{x'}$. For this we need to solve the integral in
\eq{eq:directpoly}. Without loss of generality we may assume that $x =
(x_1, x_2, 0, \ldots 0)$ and that $x' = (1, 0, \ldots 0)$ with
$\nbr{x} = \nbr{x'} = 1$. In this case a single summand of \eq{eq:directpoly} becomes
\begin{align}
  \int_{S_{d-1}} \inner{x}{v}^p \inner{x'}{v}^p dv & = 
  \int_{S_{d-1}} (v_1 x_1 + v_2 x_2)^p v_1^p dv \\
  \nonumber
  & = \sum_{i=0}^p {p \choose i} x_1^{p-i} x_2^i \int_{-1}^1  v_1^{2p - i}
  (1-v_1^2)^{\frac{i + d - 3}{2}} dv_1 \int_{S_{d-2}} v_2^i dv \\
  \nonumber
  & = \sum_{i=0}^p {p \choose i} x_1^{p-i} x_2^i \int_{-1}^1  v_1^{2p - i}
  (1-v_1^2)^{\frac{i + d - 3}{2}} dv_1 \int_{-1}^1 v_2^i
  (1-v_2^2)^{\frac{d-4}{2}} dv_2 \abr{S_{d-3}} \\
  & = \abr{S_{d-3}} \sum_{i=0}^p {p \choose i} x_1^{p-i} x_2^i 
  \frac{\Gamma\rbr{\frac{2p-i+1}{2}}\Gamma\rbr{\frac{i+d-1}{2}}}{\Gamma\rbr{\frac{2p + d}{2}}}
  \frac{\Gamma\rbr{\frac{i+1}{2}} \Gamma\rbr{\frac{d-2}{2}}}{\Gamma\rbr{\frac{i+d-1}{2}}}
\end{align}
Using the fact that $x_1 = \theta$ and $x_2 = \sqrt{1 - \theta^2}$ we
have the full expansion of \eq{eq:directpoly} via
\begin{align}
  \nonumber
  k(x,x') = \sum_p \nbr{x}^p \nbr{x'}^p c_p
  \frac{\abr{S_{d-3}}}{\abr{S_{d-1}}}
  \sum_{i=0}^p \theta^{p-i} \sbr{1 - \theta^2}^{\frac{i}{2}}
  {p \choose i}
  \frac{\Gamma\rbr{\frac{2p-i+1}{2}}\Gamma\rbr{\frac{i+d-1}{2}}}{\Gamma\rbr{\frac{2p + d}{2}}}
  \frac{\Gamma\rbr{\frac{i+1}{2}} \Gamma\rbr{\frac{d-2}{2}}}{\Gamma\rbr{\frac{i+d-1}{2}}}
\end{align}
The above form is quite different from commonly used inner-product
kernels, such as an inhomogeneous polynomial $\rbr{\inner{x}{x'} +
  d}^p$. That said, the computational savings are considerable and the 
expansion bears sufficient resemblance to warrant its use due to
significantly faster evaluation.

\section{Sampling Basis Functions}
\label{sec:sampling-bf}

\subsection{Random Kitchen Sinks}

We now discuss computationally efficient strategies for approximating
the function expansions introduced in the previous section, beginning
with Random Kitchen Sinks of \cite{RahRec08}, as described in
Section~\ref{sec:mercer}.  
Direct use for Gaussian RBF kernels yields the following algorithm to
approximate kernel functions by explicit feature construction:
\begin{algorithmic}
  \STATE {\bfseries input} Scale $\sigma^2$, $n$, $d$
  \STATE Draw $Z \in \RR^{n \times d}$ with iid entries $Z_{ij} \sim
  \Ncal(0, \sigma^{-2})$.
  \FORALL{x}
  \STATE Compute empirical feature map via 
  $\displaystyle \phi_j(x) = \frac{c}{\sqrt{n}} \exp(i [Z x]_j)$
  \ENDFOR
\end{algorithmic}
As discussed previously, and as shown by \cite{RahRec09}, the
associated feature map converges in expectation to the Gaussian RBF
kernel. Moreover, they also show that this convergence occurs with
high probability and at the rate of independent empirical
averages. While this allows one to use primal space methods, the
approach remains limited by the fact that we need to store $Z$ and,
more importantly, we need to compute $Z x$ for each $x$. That is, each
observation costs $O(n \cdot d)$ operations. This seems wasteful,
given that we are really only multiplying $x$ with a `random' matrix
$Z$, hence it seems implausible to require a high degree of accuracy
for $Z x$.

The above idea can be improved to extend matters beyond a Gaussian RBF
kernel and to reduce the memory footprint in computationally expensive
settings. We summarize this in the following two remarks:
\begin{remark}[Reduced Memory Footprint]
To avoid storing the Gaussian random matrix $Z$ we recompute $Z_{ij}$
on the fly. Assume that we have access to
a random number generator which takes samples from the uniform
distribution $\xi \sim U[0,1]$ as input and emits samples from a Gaussian,
e.g.\ by using the inverse cumulative distribution function $z =
F^{-1}(\xi)$. Then we may replace the random number generator by a
hash function via $\xi_{ij} = N^{-1} h(i,j)$ where $N$ denotes the
range of the hash, and subsequently $Z_{ij} =
F^{-1}(\xi_{ij})$. 
\end{remark}
Unfortunately this variant is computationally even more costly than
Random Kitchen Sinks, its only benefit being the $O(n)$ memory
footprint relative to the $O(n d)$ footprint for random kitchen
sinks. To make progress, a more effective approximation of the
Gaussian random matrix $Z$ is needed. 

\subsection{Fastfood}
\label{sec:gaussian-kernels}

For simplicity we begin with the Gaussian RBF case and extend it to more general
spectral distributions subsequently. Without loss of generality assume
that $d = 2^l$ for some $l \in \NN$.\footnote{If this is not the case,
  we can trivially pad the vectors with zeros until $d=2^l$ holds.}
For the moment assume that $d = n$.  The matrices that we consider
instead of $Z$ are parameterized by a product of diagonal matrices and
the Hadamard matrix:
\begin{align}
  \label{eq:fastfood}
  V := \frac 1 {\sigma\sqrt{d}}S H G \Pi H B.
\end{align}
Here $\Pi \in \cbr{0, 1}^{d \times d}$ is a permutation matrix and $H$
is the Walsh-Hadamard matrix.\footnote{We conjecture that $H$ can be replaced
by any matrix $T\in\RR^{d \times d}$, such that $T/\sqrt{d}$ is orthonormal, $\max_{ij}|T_{ij}| = O(1)$, i.e.~$T$ is smooth, and $Tx$ can be computed in $O(d\log d)$ time. A natural candidate is the Discrete Cosine Transform (DCT).} 
$S, G$ and $B$ are all \emph{diagonal}
random matrices. More specifically, $B$ has random $\cbr{\pm 1}$
entries on its main diagonal, $G$ has random Gaussian entries, and $S$
is a random scaling matrix.  $V$ is then used to compute the feature
map.

The coefficients for $S, G, B$ are computed once and stored. On the
other hand, the Walsh-Hadamard matrix is \emph{never} computed
explicitly. Instead we only multiply by it via the fast Hadamard
transform, a variant of the FFT which allows us to compute $H_d x$ in
$O(d \log d)$ time. The Hadamard matrices are defined as follows:
\begin{align*}
  H_2 := \mymatrix{rr}{
    1 & 1 \\ 1 & -1
  }
  \text{ and }
  H_{2d} := \mymatrix{rr}{
    H_d & H_d \\ H_d & -H_d
  }.
\end{align*}
When $n > d,$ we replicate \eq{eq:fastfood} for $n/d$ independent random
matrices $V_i$ and stack them via 
$V^T = [V_1, V_2, \ldots V_{n/d}]^T$
until we have enough dimensions. The feature map for Fastfood is then
defined as
\begin{align}
  \label{eq:fastfood-define}
  \phi_j(x) = n^{-\frac{1}{2}} \exp(i [V x]_j).
\end{align}
Before proving that in expectation this transform yields a Gaussian
random matrix, let us briefly verify the computational efficiency of
the method.
\begin{lemma}[Computational Efficiency]
  The features of \eq{eq:fastfood-define} can be computed at $O(n \log
  d)$ cost using $O(n)$ permanent storage for $n \geq d$. 
\end{lemma}
\begin{proof}
  Storing the matrices $S, G, B$ costs $3n$ entries and $3n$
  operations for a multiplication. The permutation matrix $\Pi$ costs
  $n$ %d$
  entries and $n$ operations. The Hadamard matrix itself requires no
  storage since it is only implicitly represented. Furthermore, the
  fast Hadamard transforms costs $O(n \log d)$ operations to carry out
  since we have $O(d \log d)$ per block and $n/d$ blocks. Computing
  the Fourier basis for $n$ numbers is an $O(n)$ operation. Hence the
  total CPU budget is $O(n \log d)$ and the storage is $O(n)$.
\end{proof}
Note that the construction of $V$ is analogous to that of
\cite{DasKumSar11}. We will use these results in establishing a
sufficiently high degree of decorrelation between rows of $V$. Also
note that multiplying with a longer chain of Walsh-Hadamard matrices
and permutations would yield a distribution closer to independent
Gaussians. However, as we shall see, two matrices provide a
sufficient amount of decorrelation. 

\subsection{Basic Properties}

Now that we showed that the above operation is \emph{fast}, let us give some
initial indication why it is also \emph{useful} and how the remaining
matrices $S, G, B, \Pi$ are defined. 
\begin{description}
\item[Binary scaling matrix $B$:] This is a diagonal matrix with
  $B_{ii} \in \cbr{\pm 1}$ drawn iid from the uniform distribution
  over $\cbr{\pm 1}$. The initial 
  $HBd^{-\frac 1 2}$ acts as an isometry that densifies the input, as
  pioneered by \cite{AilCha09}.
\item[Permutation $\Pi$:] It ensures that the rows of the
  two Walsh-Hadamard matrices are incoherent relative to each other.
  $\Pi$ can be stored efficiently as a lookup table at $O(d)$
  cost and it can be generated by sorting random
  numbers. 
\item[Gaussian scaling matrix $G$:] This is a diagonal matrix whose
  elements $G_{ii} \sim \Ncal(0, 1)$ are drawn iid from a
  Gaussian. The next Walsh-Hadamard matrices $H$ will allow us to 'recycle'
  $n$ Gaussians to make the resulting matrix closer to
  an iid Gaussian. The goal of the preconditioning steps above is to 
  guarantee that no single $G_{ii}$ can influence the output too much and
  hence provide near-independence.
\item[Scaling matrix $S$:] Note that the length of all 
  rows of $HG \Pi HB$ are constant as equation~(\ref{eq:samelength})
  shows below.
  In the Gaussian case $S$
  ensures that the length distribution of the row of $V$ are independent 
  of each other.
  In the more general case, one may also adjust the capacity of the function
  class via a suitably chosen scaling matrix $S$. That is, large
  values in $S_{ii}$ correspond to high complexity basis functions
  whereas small $S_{ii}$ relate to simple functions with low total
  variation. 
  For the RBF kernel we choose
  \begin{align}
    \label{eq:sii}
    S_{ii} = s_i \nbr{G}_\mathrm{Frob}^{-\frac{1}{2}} 
    \text{ where }
    p(s_{i}) \propto r^{d-1} e^{-\frac{r^2}{2}}.
  \end{align}
  Thus $s_i$ matches the radial part of a normal
  distribution and we rescale it using the Frobenius norm of $G$.
  % Makes no difference as only the cos(V(x-x')) part matters from the feature 
  % map and cos(a)=cos(-a).
  % The former decorrelates rows. 
\end{description}
We now analyze the distribution of entries in $V$. 
\begin{description}
\item[The rows of $HG \Pi HB$ have the same length.]
To compute their length we take
%\begin{small}
\begin{align}
    \label{eq:samelength}
    l^2 := \sbr{HG \Pi HB (HG \Pi HB)^\top}_{jj} 
    = [HG^2H]_{jj}d = \sum_i H_{ij}^2 G_{ii}^2d = \nbr{G}^2_\mathrm{Frob}d
\end{align}
In this we used the fact that $H^\top H = d \one$ and moreover that
$|H_{ij}| = 1$. Consequently, rescaling the entries by
$\nbr{G}_\mathrm{Frob}^{-\frac{1}{2}}d^{-\frac{1}{2}}$ yields rows of
length 1.
\item[Any given row of $HG \Pi HB$ is iid Gaussian.]
  Each entry of the matrix 
  $$[HG \Pi HB]_{ij} = B_{jj}H_i^TG\Pi H_j$$ 
  is zero-mean Gaussian as it consists of a sum of zero-mean
  independent Gaussian random variables.  Sign changes retain
  Gaussianity. Also note that $\var{[HG \Pi HB]_{ij}} = d$.  $B$
  ensures that different entries in $[HG\Pi HB]_{i \cdot}$ have $0$
  correlation.  Hence they are iid Gaussian (checking first and second
  order moments suffices).
\item[The rows of $SHG \Pi HB$ are Gaussian.]
  Rescaling the length of a Gaussian vector using \eq{eq:sii} retains
  Gaussianity. 
  % Sign changes leave this unchanged for zero-mean Gaussians. 
  Hence the rows of $SHG\Pi HB$ are Gaussian, albeit not independent.
\end{description}
\begin{lemma}
  \label{lem:key}
  The expected feature map recovers the Gaussian RBF kernel, i.e.,\ 
  \begin{align}
    \nonumber
    \Eb_{S,G,B,\Pi} \sbr{\overline{\phi(x)}^\top \phi(x')}
    = %\propto 
    e^{-\frac{\nbr{x-x'}^2} {2\sigma^2}}.
  \end{align}
  Moreover, the same holds for $V' = \frac 1 {\sigma\sqrt{d}}HG \Pi HB$. 
\end{lemma}
\begin{proof}
  We already proved that any given row in $V$ is a random Gaussian
  vector with distribution $\Ncal(0, \sigma^{-2}I_d)$, 
  hence we can directly appeal to the construction of
  \cite{RahRec08}. This also holds for $V'$. The main difference
  being that the rows in $V'$ are considerably more correlated. Note
  that by assembling several $d \times d$ blocks to obtain an $n
  \times d$ matrix this property is retained, since each block is
  drawn independently.
\end{proof}

\subsection{Changing the Spectrum}
\label{sec:spectrum}

Changing the kernel from a Gaussian RBF to any other radial basis
function kernel is straightforward. After all, $HG \Pi HB$ provides
a approximately spherically uniformly distributed random vectors of
the same length. Rescaling each direction of projection separately costs only
$O(n)$ space and computation. Consequently we are free to choose
different coefficients $S_{ii}$ rather than \eq{eq:sii}. Instead, we
may use
\begin{align*}
  S_{ii} \sim c^{-1} r^{d-1} A_{d-1}^{-1} \lambda(r).
\end{align*}
Here $c$ is a normalization constant and 
$\lambda(r)$ is the radial part of the spectral density function
of the regularization operator associated with the kernel.

A key advantage over a conventional kernel approach is that we are not
constrained by the requirement that the spectral distributions be
analytically computable. Even better, we only need to be able to
\emph{sample} from the distribution rather than compute its Fourier
integral in closed form.

For concreteness consider the Matern kernel. Its spectral properties
are discussed, e.g.\ by \cite{SchSmo02}. In a nutshell, given data in
$\RR^d$ denote by $\nu := \frac{d}{2}$ a dimension
calibration and let $t \in \NN$ be a fixed parameter determining the
degree of the Matern kernel (which is usually determined experimentally). 
Moreover, denote by $J_\nu(r)$ the Bessel function of the
first kind of order $\nu$. Then the kernel given by
\begin{align}
  \label{eq:maternpower}
  k(x,x') := \nbr{x-x'}^{-t \nu} J^t_\nu(\nbr{x-x'})
  \text{ for } {n \in \NN}
\end{align}
has as its associated Fourier transform
\begin{align*}
  \Fcal{k}(\omega) = {\bigotimes_{i=1}^n \chi_{S_d}}[\omega].
\end{align*}
Here $\chi_{S_d}$ is the characteristic function on the unit
ball in $\RR^d$ and $\bigotimes$ denotes convolution. 
In words, the Fourier transform of $k$ is the $n$-fold
convolution of $\chi_{S_d}$. Since convolutions of distributions arise
from adding independent random variables this yields
a simple algorithm for computing the Matern kernel:

\begin{algorithmic}
  \FOR{{\bfseries each} $S_{ii}$} 
  \STATE Draw $t$ iid samples $\xi_i$ uniformly from $S_d$.
  \STATE Use $S_{ii} = \nbr{\sum_{i=1}^t \xi_i}$ as scale.
  \ENDFOR
\end{algorithmic}
While this may appear costly, it only needs to be carried out once at
initialization time and it allows us to sidestep computing the
convolution entirely. After that we can store the coefficients
$S_{ii}$. Also note that this addresses a rather surprising problem
with the Gaussian RBF kernel --- in high dimensional spaces draws from
a Gaussian are strongly concentrated on the surface of a sphere. That
is, we only probe the data with a fixed characteristic
length. The Matern kernel, on the other hand, spreads its capacity
over a much larger range of frequencies.

\subsection{Inner Product Kernels}

We now put Theorem~\ref{th:polyexpand} and
Corollary~\ref{cor:sample-poly} to good use. Recall that the latter
states that any dot-product kernel can be obtained by taking
expectations over draws from the degree of corresponding Legendre
polynomial and over a random direction of reference, as established by
the integral representation of \eq{eq:integral}. 

It is understood that the challenging part is to draw vectors
uniformly from the unit sphere. Note, though, that it is this very
operation that Fastfood addresses by generating pseudo-Gaussian
vectors. Hence the modified algorithm works as follows:

\begin{algorithmic}
  \STATE {\bfseries Initialization}
  \FOR{$j=1$ {\bfseries to } $n/d$}
  \STATE Generate matrix block $V_j \leftarrow \nbr{G}^{-1}_\mathrm{Frob}
  d^{-\frac{1}{2}} H G \Pi H B$ implicitly as per \eq{eq:fastfood}.
  \STATE Draw degrees $n_i$ from $p(n) \propto \lambda_n N(d,n)$.
  \ENDFOR \\[2mm]
  \STATE {\bfseries Computation}
  \STATE $r \leftarrow \nbr{x}$ and $t \leftarrow V x$
  \FOR{$i=1$ {\bfseries to} $n$}
  \STATE $\psi_i \leftarrow L_{n_i,d}(t_i) = r^{n_i} L_{n_i,d}(t_i/r)$
  \ENDFOR
\end{algorithmic}
Note that the equality $L_{n_i,d}(t_i) = r^{n_i} L_{n_i,d}(t_i/r)$
follows from the fact that $L_{n_i,d}$ is a homogeneous polynomial of
degree $n_i$. The second representation may sometimes be more
effective for reasons of numerical stability. As can be seen, this
relies on access to efficient Legendre polynomial computation. Recent
work of \cite{BogMicFos12} shows that (quite surprisingly) this is
possible in $O(1)$ time for $L_n(t)$ regardless of the degree of the
polynomial. Extending these guarantees to associated Legendre
polynomials is unfortunately rather nontrivial. Hence, a direct
expansion in terms of $\inner{x}{v}^d$, as discussed previously, or
brute force computation may well be more effective. 

\begin{remark}[Kernels on the Symmetric Group]
We conclude our reasoning by providing an extension of the above
argument to the symmetric group. Clearly, by treating permutation
matrices $\Pi \in C_n$ as $d \times d$ dimensional vectors, we can use
them as inputs to a dot-product kernel. Subsequently, taking inner
products with random reference vectors of unit length yields kernels
which are dependent on the matching between permutations only. 
\end{remark}

% sufficient conditions - unit length

\section{Analysis}

The next step is to show that the feature map is well behaved also in
terms of decorrelation between rows of $V$. We focus on Gaussian RBF
kernels in this context.

\subsection{Low Variance}

When using random kitchen sinks, the variance of the feature map is at
least $O(1/n)$ since we draw $n$ samples iid from the space of
parameters. In the following we show that the variance of fastfood is
comparable, i.e.\ it is also $O(1/n)$, albeit with a dependence on the
magnitude of the magnitude of the inputs of the feature map. This
guarantee matches empirical evidence that both algorithms perform
equally well as the exact kernel expansion. 

For convenience, since the kernel values are real numbers, let us
simplify terms and rewrite the inner product in terms of a
sum of cosines. Trigonometric reformulation yields 
\begin{align}
  \frac{1}{n} \sum_j \bar\phi_j(x) \phi_j(x') = \frac{1}{n} \sum_j
  \cos {[V(x-x')]_j}
  \text{ for }
  V = d^{-\frac{1}{2}} HG\Pi H B.
\end{align}
We begin the analysis with a general variance bound for square
matrices $V \in \RR^{d \times d}$. The extension to $n/d$ iid drawn
stacked matrices is deferred to a subsequent corollary.

\begin{theorem}
  \label{thm:variance}
  Let $v=\frac{x-x'}{\sigma}$ and let $\psi_j(v)=\cos [Vv]_j$ denote the
  estimate of the kernel value arising from the $j$th pair of random
  features for each $j \in \cbr{1 \ldots d}$.  Then for each $j$ we have
  \begin{align}
    \label{eq:var-claim}
    \var\sbr{\psi_j(v)} = \frac{1}{2} \rbr{1-e^{-\nbr{v}^2}}^2
    \text{ and }
    \var\sbr{\sum_{j=1}^d\psi_j(v)} \leq 
    \frac{d}{2} \rbr{1-e^{-\nbr{v}^2}}^2 + d C(\nbr{v})
  \end{align}
  where $C(\alpha) = 6 \alpha^4 \sbr{e^{-\alpha^2} +
    \frac{\alpha^2}{3}}$ depends on the scale of the argument of the kernel.
\end{theorem}
\begin{proof}
  Since for any random variable $X_j$ we can decompose 
  $\var \rbr{\sum X_j}  = \sum_{j,t}\cov(X_j,X_t)$ 
  our goal is to compute
  \begin{align*}
    \cov(\psi(v),\psi(v)) =
    \Eb\sbr{\psi(v)\psi(v)^\top}-\Eb\sbr{\psi(v)}
    \Eb\sbr{\psi(v)}^\top.
  \end{align*}
  We decompose $V v$ into a sequence of terms $w= d^{-\frac{1}{2}}
  HBv$ and $u=\Pi w$ and $z=HGu$. Hence we have
  $\psi_j(v)=\cos(z_j)$. Note that $\nbr{u} = \nbr{v}$ since by
  construction $d^{-\frac{1}{2}} H$, $B$ and $\Pi$ are orthonormal
  matrices.

  \paragraph{Gaussian Integral}

  Now condition on the value of $u$. Then it follows that 
  $\cov(z_j,z_t | u) = \rho_{jt}(u) \nbr{v}^2 = \rho(u) \nbr{v}^2$ 
  where $\rho \in [-1,1]$ is the correlation of $z_j$ and $z_t$. By
  symmetry all $\rho_{ij}$ are identical. 

  Observe that the marginal distribution of each $z_j$ is $\Ncal(0,
  \nbr{v}^2)$ since each element of $H$ is $\pm 1$.
  Thus the joint distribution of $z_j$ and $z_t$ is a Gaussian with
  mean $0$ and covariance
  \begin{align*}
    \cov\sbr{[z_j, z_t]|u} = 
    \mymatrix{rr}{1 & \rho \\ \rho & 1} \nbr{v}^2 = L \cdot L^T
    \text{ where }
    L = \mymatrix{cc}{1 & 0 \\ \rho & \sqrt{1-\rho^2}} \nbr{v}
  \end{align*}
  is its Cholesky factor. Hence
  \begin{align}
    \label{eq:cov-expanded}
    \cov(\psi_j(v),\psi_t(v) | u) = \Eb_g \sbr{\cos([Lg]_1)\cos([Lg]_2)}  
    - \Eb_g[\cos(z_j)]\Eb_g[\cos(z_t)]
  \end{align}
  where $g \in \RR^2$ is drawn from $\Ncal(0, \one)$. From the trigonometric identity
  \begin{equation*}
    \cos(\alpha)\cos(\beta) = \frac{1}{2} \sbr{\cos(\alpha-\beta) + \cos(\alpha+\beta)}
  \end{equation*} 
  it follows that we can rewrite
  \begin{align*}
    \Eb_g\sbr{\cos([Lg]_1)\cos([Lg]_2)} 
    = \frac{1}{2} \Eb_h\sbr{\cos(a_- h) + \cos(a_+ h)} 
    = \frac{1}{2} \sbr{e^{-\frac{1}{2} a_-^2} + e^{-\frac{1}{2} a_+^2}}
  \end{align*}
  where $h \sim \Ncal(0, 1)$ and $a^2_\pm = \nbr{L^\top [1, \pm 1]}^2
  = 2 \nbr{v}^2 (1\pm\rho)$. That is, after applying the addition
  theorem we explicitly computed the now one-dimensional Gaussian
  integrals. 

  We compute the first moment analogously. Since by construction $z_j$
  and $z_j$ have zero mean and variance $\nbr{v}^2$ we have that 
  \begin{align}
    \Eb_g[\cos(z_j)] \Eb_g[\cos(z_t)] = 
    \Eb_h[\cos( \nbr{v} h)]^2 = 
    e^{-\nbr{v}^2}
    \nonumber
  \end{align}
  Combining both terms we obtain that the covariance can be written as
  \begin{align}
    \label{eq:covtemplate}
    \cov[\psi_j(v), \psi_t(v)|u] = e^{-\nbr{v}^2} \sbr{\cosh[\nbr{v}^2\rho] - 1}
  \end{align}

  \paragraph{Taylor Expansion}

  To prove the first claim note that here $j = t$, since we are
  computing the variance of a single feature. Correspondingly $\rho(u)
  = 1$. Plugging this into \eq{eq:covtemplate} and simplifying terms
  yields the first claim of \eq{eq:var-claim}.

  To prove our second claim, observe that from the Taylor series of
  $\cosh$ with remainder in Lagrange form, it follows that there
  exists $\eta \in [-\nbr{v}^2|\rho|,\nbr{v}^2|\rho|]$ such
  \begin{align*}
    \cosh(\nbr{v}^2\rho) - 1 &= \frac{1}{2} \nbr{v}^4\rho^2 +
    \frac{1}{6} \sinh(\eta)\nbr{v}^6\rho^3 \\
    &\leq
    \frac{1}{2} \nbr{v}^4\rho^2 + \frac{1}{6} \sinh(\nbr{v}^2)\nbr{v}^6\rho^3 \\
    &\leq \rho^2 \nbr{v}^4 B(\nbr{v}),
  \end{align*}
  where $B(\nbr{v})= \frac 1 2 + \frac{\sinh(\nbr{v}^2)\nbr{v}^2} 6$.
  Plugging this into \eq{eq:covtemplate} yields
  $$\cov[\psi_j(v),\psi_t(v) | u] \le \rho^2 \nbr{v}^4 B(\nbr{v}).$$

  \paragraph{Bounding $\Eb_u[\rho^2]$}

  Note that the above is still conditioned on $u$.  What remains is to bound
  $\Eb_u[\rho^2]$, which is small if $\Eb[\nbr{u}_4^4]$ is small. The
  latter is ensured by $HB$, which acts as a randomized
  preconditioner:   
  Since $G$ is diagonal and $G_{ii} \sim \Ncal(0,1)$ independently we
  have
  \begin{align*}
    \cov[z,z] = \cov[HGu, HGu] = H\cov[Gu, Gu]H^\top 
    = H\Eb\sbr{\mathrm{diag}(u_1^2,\ldots,u_d^2)} H^\top.
  \end{align*}
  Recall that $H_{ij} = H_{ji}$ are elements of the Hadamard
  matrix. For ease of notation fix $j \ne t$ and let $T=\{i \in [1..d]
  : H_{ji}=H_{ti}\}$ be the set of columns where the $j^\mathrm{th}$
  and the $t^\mathrm{th}$ row of the Hadamard matrix agree. Then
  \begin{align*}
    \cov(z_j,z_t | u) = \sum_{i=1}^d H_{ji} H_{ti} u_i^2 = 
    \sum_{i \in T}u_i^2 - \sum_{i \notin T}u_i^2 
    = 2\sum_{i \in T}u_i^2 - \sum_{i=1}^du_i^2 = 2\sum_{i \in T}u_i^2 - \nbr{v}^2.
  \end{align*}
Now recall that $u=\Pi w$ and that $\Pi$ is a random permutation matrix. Therefore 
$u_i=w_{\pi(i)}$ for a randomly chosen permutation $\pi$ and thus the distribution of $\rho\nbr{v}^2$ 
and $2\sum_{i \in R}w_i^2 - \nbr{v}^2$ where $R$ is a randomly chosen
subset of size $\frac{d}{2}$ in 
$\cbr{1\ldots d}$ are the same. Let us fix (condition on) $w$. 
Since $2\Eb_{R}\sbr{\sum_{i \in R}w_i^2}=\nbr{v}^2$ we have that
\begin{equation}
  \nonumber
\Eb_{R}\sbr{\rho^2\nbr{v}^4}=4\Eb_{R}\sbr{\sbr{\sum_{i \in R}w_i^2}^2} -\nbr{v}^4.
\end{equation}
Now let $\delta_i=1$ if $i \in R$ and $0$ otherwise. Note that
$\Eb_{\delta}(\delta_i) = \frac 1 2$ and if $j \ne k $ then
$\Eb_{\delta}(\delta_i\delta_k) \le \frac 1 4$ as $\delta_i$ are
(mildly) negatively correlated.  From $\nbr{w}=\nbr{v}$ it follows
that
\begin{align}
  \nonumber
\Eb_R\sbr{\sbr{\sum_{i \in
      R}w_i^2}^2}=\Eb_{\delta}\sbr{\sbr{\sum_{i=1}^d\delta_iw_i^2}^2}
=   
\Eb_{\delta}\sbr{\sum_{i\ne k}\delta_i\delta_kw_i^2w_k^2}+\Eb_{\delta}{\sum_{i}\delta_iw_i^4} \le 
\frac {\nbr{v}^4} 4 + \frac {\nbr{w}_4^4} 2.
\end{align}
From the two equations above it follows that
\begin{equation}
\label{eq:4th-moment-3}
\Eb_R\sbr{\rho^2 \nbr{v}^4} \le 2 {\nbr{w}_4^4}.
\end{equation}

\paragraph{Bounding the fourth moment of $\nbr{w}$}

Let $b_i = B_{ii}$ be the independent $\pm 1$ random variables of
$B$. 
Using the fact $w_i=\frac 1 {\sqrt{d}}\sum_{t=1}^d H_{it} b_t v_t$ and that $b_i$ are independent
with similar calculations to the above it follows that
\begin{align*}
  \Eb_{b}\sbr{w_i^4} \le \frac{6}{d^2} \sbr{v_i^4 + \sum_{t \ne j} v_t^2v_j^2}
  \text{ and hence }
  \Eb_{b}\sbr{\nbr{w}_4^4} \le \frac{6}{d} \nbr{v}_2^4
\end{align*}
which shows that $\frac 1 {\sqrt d} HB$ acts as preconditioner
that densifies the input.
Putting it all together we have 
\begin{align*}
\sum_{j \ne t}\Eb_u\sbr{\cov(\psi_j(v),\psi_t(v) | u)} 
& \le 
 d^2 e^{-\nbr{v}^2}B(\nbr{v}) \Eb_R[\rho^2\nbr{v}^4]
\le 
 12 d e^{-\nbr{v}^2} B(\nbr{v})\nbr{v}^4 \\
& = 
 6 d \nbr{v}^4 \rbr{e^{-\nbr{v}^2} + {\nbr{v}^2}/3} 
\end{align*}
Combining the latter with the already proven first claim establishes the second claim.
\end{proof}

\begin{corollary}
  \label{thm:variance-n}
  Denote by $V, V'$ Gauss-like matrices of the form
  \begin{align}
    V = \sigma^{-1} d^{-\frac{1}{2}} HG\Pi HB
    \text{ and }
    V' = \sigma^{-1} d^{-\frac{1}{2}} S HG\Pi HB. 
  \end{align}
  Moreover, let 
  $C(\alpha) = 6 \alpha^4 \sbr{e^{-\alpha^2} + \frac{\alpha^2}{3}}$ be
  a scaling function. 
  Then for the feature maps obtained by stacking $n/d$ iid copies of
  either $V$ or $V'$ we have
  \begin{align}
    \label{eq:var-n}
    \var \sbr{\phi'(x)^\top \phi'(x')}
    \le \frac{2}{n} {\rbr{1-e^{-\nbr{v}^2}}^2} + \frac{1}{n} {C(\nbr{v})} 
    \text{ where } 
    v = \sigma^{-1} (x-x').
  \end{align}
\end{corollary}
\begin{proof}
  Since $\phi'(x)^\top \phi'(x')$ is the average of $n/d$ independent
  estimates, each arising from $2d$ features. Hence we can appeal to
  Theorem~\ref{thm:variance} for a single block, i.e.\ when $n=d$. The
  near-identical argument for $V$ is omitted.
\end{proof}

\subsection{Concentration}

The following theorem shows that for a given error probability
$\delta$, the approximation error of a $d \times d$ block of Fastfood
is at most logarithmically larger than the error of
Random Kitchen Sinks. That is, it is only logarithmically
weaker. We believe that this bound is pessimistic and
could be further improved with considerable analytic effort. That
said, the $O(m^{-\frac{1}{2}})$
approximation guarantees to the kernel matrix are likely rather
conservative when it comes to generalization performance, as we found
in experiments. In other words, we found that the algorithm works much
better in practice than in theory, as confirmed in
Section~\ref{sec:exp}. Nonetheless it is important to establish tail
bounds, not to the least since this way improved guarantees for random
kitchen sinks also immediately benefit fastfood. 

%  (likewise
% we defer analyzing the concentration of $n > d$ stacked Fastfood
% features to future work).

\begin{theorem}
\label{thm:concentration}
For all $x, x'\in \RR^d$ let $\hat{k}(x,x')=\sum_{j=1}^d\cos(d^{-\frac 1 2} [HG\Pi HB(x-x')/\sigma]_j)/d$ 
denote our estimate of the RBF kernel $k(x,x')$ that arises from a $d \times d$ block of 
Fastfood. Then we have that
\begin{equation*}
\Pb\sbr{
   \abr{\hat{k}(x,x')-k(x,x')} \ge 2 \sigma^{-1} d^{-\frac{1}{2}}
   \nbr{x - x'} 
   \sqrt{\log (2/\delta) \log(2d/\delta)}} \le 2\delta
 \text{ for all } \delta > 0
\end{equation*}
\end{theorem}
Theorem~\ref{thm:concentration} demonstrates almost sub-Gaussian convergence 
Fastfood kernel for a fixed pair of points $x, x'$. A standard $\epsilon$-net argument
then shows uniform convergence over any compact set of $\RR^d$ with bounded 
diameter~\cite[Claim 1]{RahRec08}. Also, the small error of the approximate kernel 
does not significantly perturb the solution returned by wide range of learning 
algorithms~\cite[Appendix B]{RahRec08} or affect their generalization error.

Our key tool is concentration of Lipschitz continuous functions under
the Gaussian measure~\cite{Ledoux96}. We ensure that Fastfood
construct has a small Lipschitz constant using
the following lemma, which is due to \cite{AilCha09}. 

\begin{lemma}[\citealp*{AilCha09}]
\label{lem:densification}
Let $x \in \RR^d$ and $t > 0$. Let $H, B \in\RR^{d\times d}$
denote the Hadamard and the binary random diagonal matrices
respectively in our construction. Then 
%% $\Pb\sbr{ \nbr{d^{-\frac 1 2}HBx}_{\infty} \ge t\nbr{x}_2} \le 2d e^{-t^2 d/2}$.
%% In particular, 
for any $\delta > 0$ we have that 
\begin{align}
  \Pb\sbr{ \nbr{HBx}_{\infty} 
  \ge \nbr{x}_2 \sqrt{2 \log {2d}/\delta}}
  \le \delta
\end{align}
\end{lemma}
In other words, with high probability, the largest elements of
$d^{-\frac{1}{2}} H B x$ are with high probability no larger than what
one could expect if all terms were of even size as per the $\nbr{x}_2$
norm. 

To use concentration of the Gaussian measure we need Lipschitz
continuity. We refer to a function $f:\RR^d \rightarrow \RR$ as
Lipschitz continuous with constant $L$ 
if for all $x, y \in \RR^d$ it holds that $|f(x)-f(y)|\le L
\nbr{x-y}_2$. Then the following holds \cite[Inequality 2.9]{Ledoux96}:
\begin{theorem}
\label{thm:Lip-Gauss}
Assume that $f:\RR^d \rightarrow \RR$ is Lipschitz continuous with
constant $L$ and let $g \sim \Ncal(0, \one)$ be drawn from a
$d$-dimensional Normal distribution. Then we have
\begin{align}
  \Pb\sbr{\abr{f(g)-\Eb_g\sbr{f(g)}}\ge t} \le 2 e^{-\frac{t^2}{2L^2}}.
\end{align}
\end{theorem}

\begin{proof}{\bf[Theorem \ref{thm:concentration}]}
Since both $k$ and $\hat{k}$ are shift invariant we set $v=\sigma(x-x')$ and write
$k(v)=k(x,x')$ and $\hat{k}(v)=\hat{k}(x,x')$ to simplify the notation.
Set $u=\Pi d^{-\frac 1 2}HBv$, and $z=HGu$ and 
define 
$$f(G,\Pi,B)= d^{-1} \sum_{j=1}^d\cos(z_j).$$
Observe that Lemma~\ref{lem:key} implies
$\Eb_{G,\Pi,B}\sbr{f(G,\Pi,B)}=k(v)$.  Therefore it is sufficient to
prove that $f(G,\Pi,B)$ concentrates around its mean.  We will
accomplish this by showing that $f$ is Lipschitz continuous as a
function of $G$ for most $\Pi$ and $B$. For $a\in\RR^d$ let
\begin{equation}
\label{eq:h-fun-def}
  h(a)= d^{-1} \sum_{j=1}^d\cos(a_j).
\end{equation}
Using the fact that cosine is Lipschitz continuous with constant $1$ 
we observe that for any pair of vectors $a, b\in\RR^d$ it holds that
\begin{align}
  |h(a)-h(b)| \le  d^{-1} \sum_{j=1}^d|\cos(a_j)-\cos(b_j)| \leq
  d^{-1} \nbr{a-b}_1 \le d^{-\frac{1}{2}} \nbr{a-b}_2.
\label{eq:h-Lip} 
\end{align}
For any vector $g \in \RR^d$ let $\mbox{diag}(g)\in\RR^{d \times d}$
denote the diagonal matrix whose diagonal is $g$. Observe that for any
pair of vectors $g, g'\in\RR^{d}$ we have that
\begin{align}
\nonumber
  \nbr{H\mbox{diag}(g)u-H\mbox{diag}(g')u}_2  & \le
  \nbr{H}_2\nbr{\mbox{diag}(g-g')u}_2
  \label{eq:HGu-Lip}
\end{align}
Let $G=\mbox{Diag}(g)$ in the Fastfood construct and recall the
definition of function $h$ as in \eq{eq:h-fun-def}. 
Combining the above inequalities for any pair of vectors 
$g, g'\in\RR^d$ yields the following bound
\begin{equation}
\label{eq:h-HGu-Lip}
|h(H\mbox{Diag}(g)u))-h(H\mbox{Diag}(g')u)| \le \nbr{u}_{\infty}\nbr{g-g'}_2.
\end{equation}
From $u=\Pi d^{-\frac 1 2}HBv$ and $\nbr{\Pi w}_{\infty}=\nbr{w}_{\infty}$ combined with 
Lemma~\ref{lem:densification} it follows that
\begin{equation}
\label{eq:u-dense}
  \nbr{u}_{\infty} \le \nbr{v}_2\sqrt{\frac{2}{d} \log \frac {2d} {\delta}}
\end{equation}
holds with probability at least $1-\delta$, where the probability is
over the choice of $B$.\footnote{Note that in contrast to
  Theorem~\ref{thm:variance}, the permutation matrix $\Pi$ does not
  play a role in the proof of Theorem~\ref{thm:concentration}.}  Now condition
on~\eq{eq:u-dense}.  From inequality~\eq{eq:h-HGu-Lip} we have
that the function 
$$g \rightarrow h(H\mbox{diag}(g)u) = f(\mbox{diag}(g),\Pi,B)$$ 
is Lipschitz continuous with Lipschitz
constant
\begin{equation}
\label{eq:f-Lip-const}
L=\nbr{v}_2\sqrt{\frac{2}{d} \log \frac {2d}{\delta}}.
\end{equation}
Hence from Theorem~\ref{thm:Lip-Gauss} and from the independently chosen 
$G_{jj} \sim \Ncal(0,1)$ it follows that
\begin{equation}
\label{eq:f-conc}
   \Pb_G\sbr{
   \abr{f(G,\Pi,B)-k(v)} \ge \sqrt{2 L \log 2/\delta}} \le \delta.
\end{equation}
Combining inequalities~(\ref{eq:f-Lip-const}) and (\ref{eq:f-conc})
with the union bound concludes the proof.
\end{proof}

\section{Experiments}
\label{sec:exp}

In the following we assess the performance of Random Kitchen Sinks and
Fastfood. The results show that Fastfood performs as well as Random Kitchen
Sinks in terms of accuracy. Fastfood, however, is orders of magnitude
faster and exhibits a significantly lower memory footprint. For
simplicity, we focus on penalized least squares regression since in
this case we are able to compute exact solutions and are independent
of any other optimization algorithms.  We also benchmark Fastfood on
CIFAR-10~\citep{Krizhevsky09} and observe that it achieves
state-of-the-art accuracy. This advocates for the use of non-linear
expansions even when $d$ is large.

\subsection{Approximation quality}

We begin by investigating how well our features can approximate the
exact kernel computation as $n$ increases. For that purpose, we
uniformly sample 4000 vectors from $[0,1]^{10}$. We compare the exact
kernel values to Random Kitchen Sinks and Fastfood.

The results are shown in Figure~\ref{fig:convergence}. We used the
absolute difference between the exact kernel and the approximation to
quantify the error (the relative difference also exhibits similar
behavior and is thus not shown due to space constraints). The results
are presented as averages, averaging over 4000 samples. As can be
seen, as $n$ increases, both Random Kitchen Sinks and Fastfood
converge quickly to the exact kernel values. Their performance is
indistinguishable, as expected from the construction of the algorithm.

\begin{figure}[tb]
\centering
\includegraphics[width=0.8\columnwidth]{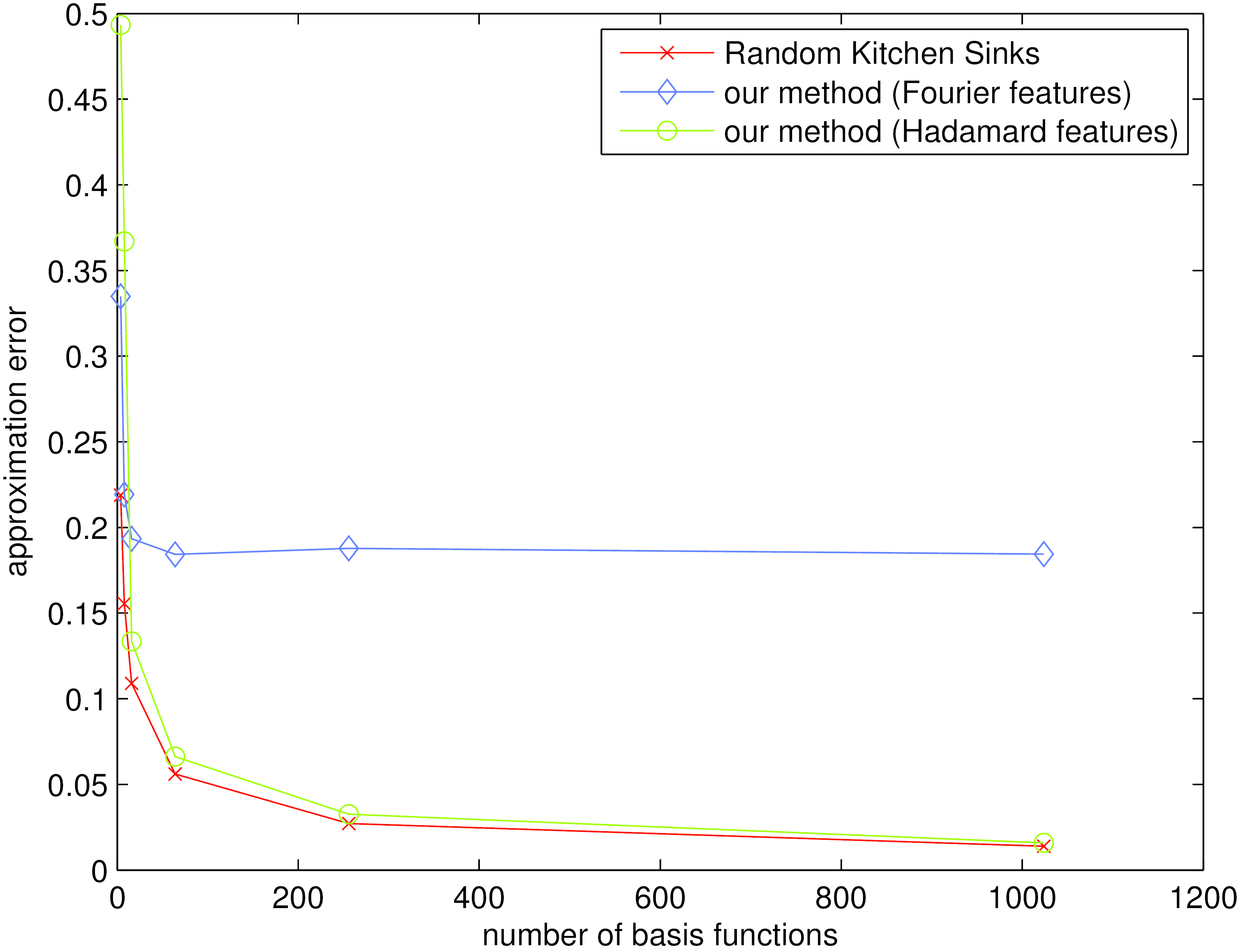}
\caption{Kernel approximation errors of different methods with respect
  to number of basis functions $n$.
\label{fig:convergence}}
\end{figure}

Note, though, that fidelity in approximating $k(x,x')$ does not 
imply generalization performance (unless the bounds are very
tight). To assess this, we carried out experiments on 
\emph{all regression datasets} from the UCI repository \citep{FraAsu10}
that are not too tiny, i.e.,\ that contained at least $4,000$ instances.

%This lower bound is reasonable since it makes no sense to approximate
%kernel functions when the inner-product matrix uses less than 64MB of
%memory.

We investigate estimation accuracy via Gaussian process
regression \citep{RasWil06} using approximated kernel computation methods and we
compare this to exact kernel computation whenever the latter is
feasible. For completeness, we compare the following methods:
\begin{description}
\item[Exact RBF] uses the exact Gaussian RBF kernel, that is $k(x,x')
  = \exp\rbr{-\nbr{x-x'}^2/2 \sigma^2}$. This is possible, albeit not
  practically desirable due to its excessive cost, on all but the
  largest datasets where the kernel matrix does not fit into memory.
\item[Nystrom] uses the Nystrom approximation of the kernel matrix 
  \citep{WilSee01}. These methods have received recent interest due to
  the improved approximation guarantees of \cite{Jinetal11} which
  indicate that approximation rates faster than $O(n^{-\frac{1}{2}})$
  are achievable. Hence, theoretically, the Nystrom method could have
  a significant accuracy advantage over Random Kitchen Sinks and
  Fastfood when using the same number of basis functions, albeit at 
  exponentially higher cost of $O(d)$ vs.\ $O(\log d)$ per
  function. We set $n = 2,048$ to retain a computationally feasible
  feature projection. 
\item[Random Kitchen Sinks] uses the the Gaussian random projection
  matrices of \cite{RahRec08}. As before, we use $n=2,048$ basis
  functions. Note that this is a rather difficult setting for Random
  Kitchen Sinks relative to the Nystrom decomposition, since the basis
  functions obtained in the latter are arguably better in terms of
  approximating the kernel. Hence, one would naively expect slightly
  inferior performance from Random Kitchen Sinks relative to direct
  Hilbert Space methods. 
\item[Fastfood] (\emph{Hadamard features}) uses the random
  matrix given by $SHG \Pi HB$, again with $n = 2,048$
  dimensions. Based on the above reasoning one would expect that the
  performance of the Hadamard features is even weaker than that of
  Random Kitchen Sinks since now the basis functions are no longer
  even independently drawn from each other. 
\item[FFT Fastfood] (\emph{Fourier features}) uses a variant of the
  above construction. Instead of combining two Hadamard matrices, a
  permutation and Gaussian scaling, we use a permutation in
  conjunction with a Fourier Transform matrix $F$: the random matrix
  given by $V=\Pi F B$. The motivation is the Subsampled Random
  Fourier Transform, as described by \cite{Tropp10}: by picking a
  random subset of columns from a (unitary) Fourier matrix, we end up
  with vectors that are almost spatially isotropic, albeit with
  slightly more dispersed lengths than in Fastfood. We use this
  heuristic for comparison purposes.
\item[Exact Poly] uses the exact polynomial kernel, that is $k(x,x') =
  (\inner{z}{x} + 1)^d$, with $d = 10$. Similar to the case of Exact
  RBF, this method is only practical on small datasets.
\item[Fastfood Poly] uses the Fastfood trick via Spherical Harmonics
  to approximate the polynomial kernels.
\end{description}
The results of the comparison are given in
Table~\ref{tab:accuracy}. As can be seen, and contrary to the
intuition above, there is virtually no difference between the exact
kernel, the Nystrom approximation, Random Kitchen Sinks and
Fastfood. In other words, Fastfood performs just as well as the exact
method, while being substantially cheaper to compute. Somewhat
surprisingly, the Fourier features work very well. This indicates that
the concentration of measure effects impacting Gaussian RBF kernels
may actually be counterproductive at their extreme. This is
corroborated by the good performance observed with the Matern kernel.

\begin{table}[tbh]
\caption{Speed and memory improvements of Fastfood relative
  to Random Kitchen Sinks
  \label{tab:improvements}}
\smallskip
\centering
  \begin{tabular}{rr|rrr|r} 
    $d$ & $n$                 &  Fastfood & RKS & Speedup &                RAM \\\hline
    $1,024$ & $16,384$        &  0.00058s & 0.0139s & 24x         & 256x  \\
    $4,096$ & $32,768$        &  0.00137s & 0.1222s & 89x         & 1024x \\
    $8,192$ & $65,536$        &  0.00269s & 0.5351 & 199x            & 2048x 
  \end{tabular}
\end{table}

\begin{sidewaystable}[tbh]
  \caption{Test set RMSE of different kernel computation methods. We
    can see Fastfood methods perform comparably with Exact RBF,
    Nystrom, Random Kitchen Sinks (RKS) and Exact Polynomial (degree
    10). $m$ and $d$ are the size of the training set the dimension of
    the input. Note that the problem size made it impossible to
    compute the exact solution for datasets of size 40,000 and up.}
  \label{tab:accuracy}
  \begin{center}
    \begin{tabular}{l|r|r|rrr|rrrr|r|r} 
Dataset               & $m$ & $d$ & Exact   & \hspace{-2mm} Nystrom &
RKS & Fastfood   &  \hspace{-2mm} Fastfood & Exact & \hspace{-2mm} Fastfood                                              & Exact & Fastfood\\ 
                      &  & & RBF              & RBF         & RBF&  FFT &  RBF & Matern & Matern                         & Poly  & Poly\\\hline
Insurance    & $5,822$ & $85$ & 0.231        & 0.232   & 0.266         & 0.266           & 0.264 & 0.234 & 0.235         & 0.256 & 0.271\\ \hline
Wine          & $4,080$ & $11$ & 0.819        & 0.797   & 0.740         & 0.721           & 0.740 &  0.753 & 0.720       & 0.827 & 0.731 \\
Quality       &         &      &              &         &               &                 &       &        &             &               \\ \hline
Parkinson            & $4,700$ & $21$ & 0.059        & 0.058   & 0.054         & 0.052           & 0.054 & 0.053 & 0.052 & 0.061 & 0.055      \\ \hline        
CPU                   & $6,554$ & $21$ & 7.271        & 6.758   & 7.103         & 4.544           & 7.366 & 4.345 & 4.211& 7.959 & 5.451\\ \hline
CT slices \hspace{-3mm} & $42,800$ & $384$ & n.a.  & 60.683 & 49.491 & 58.425 & 43.858 & n.a. & 14.868                   &  n.a. & 53.793\\
(axial)       &         &      &              &         &               &                 &       &        &             &         \\ \hline
KEGG \hspace{-3mm} & $51,686$ &  27 & n.a. & $17.872$ & $17.837$ & $17.826$ & $17.818$ &  n.a. & 17.846                  &  n.a. & 18.032\\ 
Network    &         &      &              &         &               &                 &       &        &                &         \\ \hline
Year  \hspace{-3mm}   & $463,715$ & $90$ & n.a.    &  0.113 & 0.123 & 0.106 & 0.115  & n.a. & 0.116                      &  n.a. & 0.114\\ 
Prediction       &         &      &              &         &               &                 &       &        &          &         \\ \hline
Forest                & $522,910$ & $54$ & n.a.    & 0.837   & 0.840         & 0.838           & 0.840  &  n.a. &  0.976 & n.a. & 0.894 
\end{tabular}
\end{center}
\end{sidewaystable}

In Figure~\ref{fig:expansion}, we show regression performance as a
function of the number of basis functions $n$ on the CPU dataset. As
is evident, it is necessary to have a large $n$ in order to learn
highly nonlinear functions. Interestingly, although the Fourier
features do not seem to approximate the Gaussian RBF kernel, they
perform well compared to other variants and improve as $n$
increases. This suggests that learning the kernel by direct spectral
adjustment might be a useful application of our proposed method.

\begin{figure}[tb]
\centering
\includegraphics[width=0.7\columnwidth]{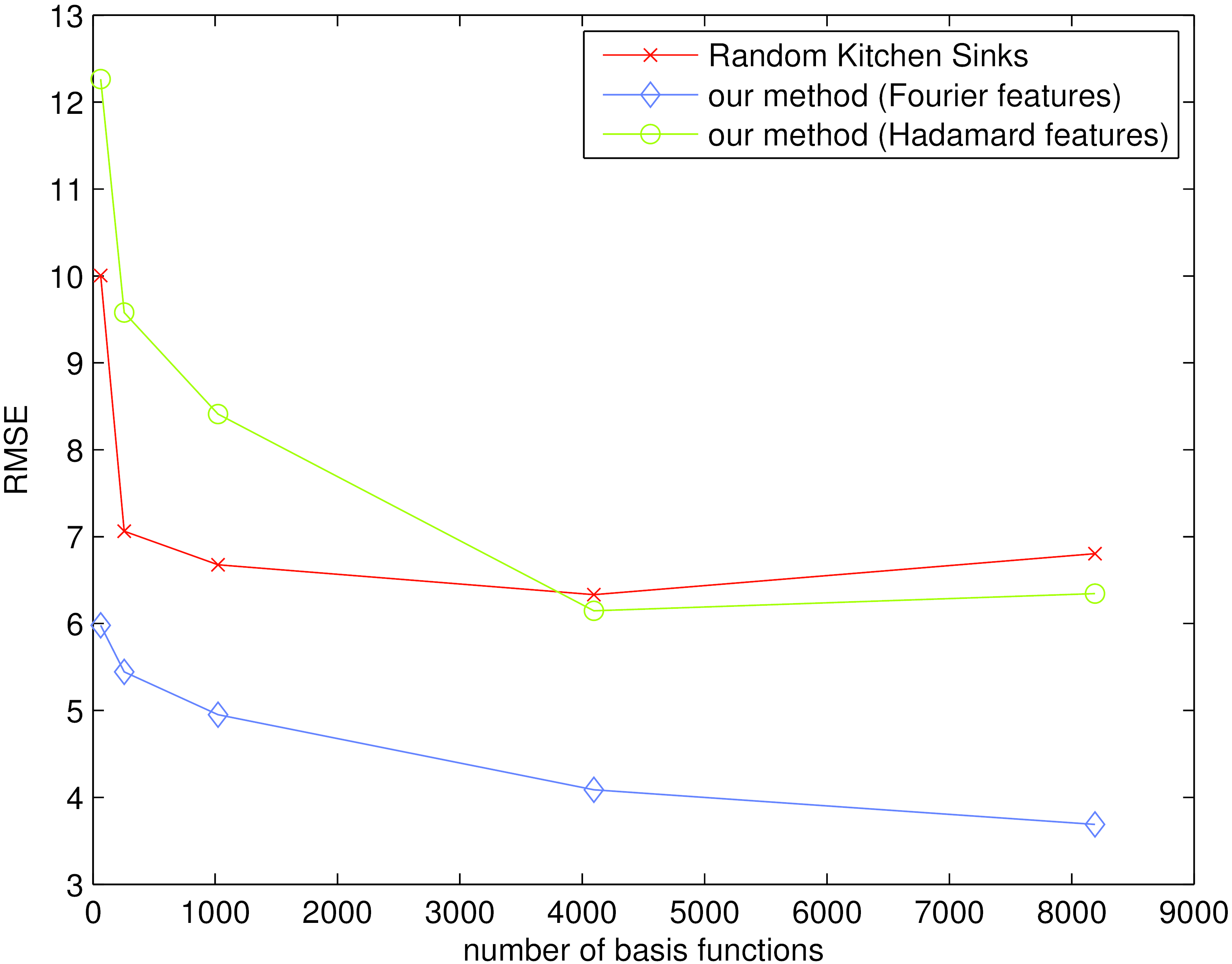}
\caption{Test RMSE on CPU dataset with respect to the number of basis
  functions. As number of basis functions increases, the quality of
  regression generally improves. 
%Note that although Fourier features
%  do not faithfully approximate the kernel RBF (see Figure 1), its
%  regression performance also improves as the number of basis function
%  increases.
}
\label{fig:expansion}
\end{figure}

\subsection{Speed of kernel computations}

In the previous experiments, we observe that Fastfood is on par with
exact kernel computation, the Nystrom method, and Random Kitchen
Sinks. The key point, however, is to establish whether the algorithm
offers computational savings. 

For this purpose we compare Random Kitchen Sinks using
Eigen\footnote{\url{http://eigen.tuxfamily.org/index.php?title=Main_Page}}
and our method using Spiral\footnote{\url{http://spiral.net}}.
Both are highly optimized numerical linear algebra libraries in
C++. We are interested in the time it takes to go from raw features of
a vector with dimension $d$ to the label prediction of that vector. On
a small problem with $d=1,024$ and $n=16,384$, performing prediction
with Random Kitchen Sinks takes 0.07 seconds. Our method is around 24x
faster, taking only 0.003 seconds to compute the label for one input
vector. The speed gain is even more significant for larger problems,
as is evident in Table~\ref{tab:improvements}.  This confirms
experimentally the $O(n \log d)$ vs.\ $O(nd)$ runtime and the $O(n)$ vs.\
$O(nd)$ storage of Fastfood relative to Random Kitchen Sinks. In other
words, the computational savings are substantial for large input
dimensionality $d$.

\subsection{Random features for CIFAR-10}

To understand the importance of nonlinear feature expansions for a
practical application, we benchmarked Fastfood, Random Kitchen Sinks on
the CIFAR-10 dataset~\cite{Krizhevsky09} which has 50,000 training images
and 10,000 test images. Each image has 32x32 pixels and 3 channels
($d = 3072$). In our experiments, linear SVMs achieve 42.3\% accuracy
on the test set. Non-linear expansions improve the classification
accuracy significantly. In particular, Fastfood FFT (``Fourier
features'') achieve 63.1\% while Fastfood (``Hadamard features'') and
Random Kitchen Sinks achieve 62.4\% with an expansion of $n =
16,384$. These are also best known classification accuracies using
permutation-invariant representations on this dataset. In terms of
speed, Random Kitchen Sinks is 5x slower (in total training time) and
20x slower (in predicting a label given an image) compared to both Fastfood
and and Fastfood FFT. This demonstrates that non-linear expansions are
needed even when the raw data is high-dimensional, and that Fastfood
is more practical for such problems. 

In particular, in many cases, linear function classes are used because
they provide fast training time, and especially test time, but not
because they offer better accuracy. The results on CIFAR-10
demonstrate that Fastfood can overcome this obstacle.

%\subsection{TODO}

%We need a few more experiments to make this a solid paper. Not much
%needed but a few more. 
%\begin{itemize}
%\item For Gaussian RBF kernel add comparison between local and
%  nonlocal basis functions.
%\item For dot product kernels compare $(\inner{x}{x'} + c)^d$ to a
%  corresponding sampling from inner product features space. I don't
%  think that using the exact associated polynomials is really
%  needed. That said, some libraries have them, e.g.\ Boost
%  \url{http://nebel.rc.fas.harvard.edu/mjuric/std/w12/external/boost/1.47.0/boost_1_47_0/libs/math/doc/sf_and_dist/html/math_toolkit/special/sf_poly/legendre.html}
%  and
%  \url{http://people.sc.fsu.edu/~jburkardt/cpp_src/legendre_polynomial/legendre_polynomial.html}. 
%\item Pick a dataset where we can beat the state of the art quite a
%  bit by using fastfood in terms of accuracy (i.e.\ by using an
%  overwhelming number of basis functions).
%\end{itemize}

\section{Summary}

We demonstrated that it is possible to compute $n$ nonlinear basis
functions in $O(n \log d)$ time, a significant speedup over the best
competitive algorithms. This means that kernel
methods become more practical for problems that have large datasets
and/or require real-time prediction. In fact, Fastfood can be used to
run on cellphones because not only it is fast, but it also requires
only a small amount of storage.

Note that our analysis is not limited to translation invariant kernels
but it also includes inner product formulations. This means that for
most practical kernels our tools offer an easy means of making kernel
methods scalable beyond simple subspace decomposition
strategies. Extending our work to other symmetry groups is subject to
future research. Also note that fast multiplications with
near-Gaussian matrices are a key building block of many randomized
algorithms. It remains to be seen whether one could use the proposed
methods as a substitute and reap significant computational savings.

%\bibliography{../../../bibfile/bibfile.bib}

\begin{thebibliography}{55}
\providecommand{\natexlab}[1]{#1}
\providecommand{\url}[1]{\texttt{#1}}
\expandafter\ifx\csname urlstyle\endcsname\relax
  \providecommand{\doi}[1]{doi: #1}\else
  \providecommand{\doi}{doi: \begingroup \urlstyle{rm}\Url}\fi

\bibitem[Ailon and Chazelle(2009)]{AilCha09}
N.~Ailon and B.~Chazelle.
\newblock The fast johnson-lindenstrauss transform and approximate nearest
  neighbors.
\newblock \emph{{SIAM} Journal on Computing}, 39\penalty0 (1):\penalty0
  302--322, 2009.

\bibitem[Aizerman et~al.(1964)Aizerman, Braverman, and Rozono\'er]{AizBraRoz64}
M.~A. Aizerman, A.~M. Braverman, and L.~I. Rozono\'er.
\newblock Theoretical foundations of the potential function method in pattern
  recognition learning.
\newblock \emph{Autom. Remote Control}, 25:\penalty0 821--837, 1964.

\bibitem[Aronszajn(1944)]{Aronszajn44}
N.~Aronszajn.
\newblock La th\'eorie g\'en\'erale des noyaux r\'eproduisants et ses
  applications.
\newblock \emph{Proc.\ Cambridge Philos.\ Soc.}, 39:\penalty0 133--153, 1944.

\bibitem[Berg et~al.(1984)Berg, Christensen, and Ressel]{BerChrRes84}
C.~Berg, J.~P.~R. Christensen, and P.~Ressel.
\newblock \emph{Harmonic Analysis on Semigroups}.
\newblock Springer, New York, 1984.

\bibitem[Bogaert et~al.(2012)Bogaert, Michiels, and Fostier]{BogMicFos12}
I.~Bogaert, B.~Michiels, and J.~Fostier.
\newblock {$O(1)$} computation of legendre polynomials and gauss--legendre
  nodes and weights for parallel computing.
\newblock \emph{SIAM Journal on Scientific Computing}, 34\penalty0
  (3):\penalty0 C83--C101, 2012.

\bibitem[Boser et~al.(1992)Boser, Guyon, and Vapnik]{BosGuyVap92}
B.~Boser, I.~Guyon, and V.~Vapnik.
\newblock A training algorithm for optimal margin classifiers.
\newblock In D.~Haussler, editor, \emph{Proc.\ Annual Conf.\ Computational
  Learning Theory}, pages 144--152, Pittsburgh, PA, July 1992. ACM Press.

\bibitem[Boyd et~al.(2010)Boyd, Parikh, Chu, Peleato, and
  Eckstein]{BoyParChuPelEtal10}
S.~Boyd, N.~Parikh, E.~Chu, B.~Peleato, and J.~Eckstein.
\newblock Distributed optimization and statistical learning via the alternating
  direction method of multipliers.
\newblock \emph{Foundations and Trends in Machine Learning}, 3\penalty0
  (1):\penalty0 1--123, 2010.

\bibitem[Burges(1996)]{Burges96}
C.~J.~C. Burges.
\newblock Simplified support vector decision rules.
\newblock In L.~Saitta, editor, \emph{Proc.\ Intl.\ Conf.\ Machine Learning},
  pages 71--77, San Mateo, CA, 1996. Morgan Kaufmann Publishers.

\bibitem[Cortes and Vapnik(1995)]{CorVap95}
C.~Cortes and V.~Vapnik.
\newblock Support vector networks.
\newblock \emph{Machine Learning}, 20\penalty0 (3):\penalty0 273--297, 1995.

\bibitem[Das and Kempe(2011)]{DasKem11}
A.~Das and D.~Kempe.
\newblock Submodular meets spectral: Greedy algorithms for subset selection,
  sparse approximation and dictionary selection.
\newblock In L.~Getoor and T.~Scheffer, editors, \emph{Proceedings of the 28th
  International Conference on Machine Learning, {ICML}}, pages 1057--1064.
  Omnipress, 2011.

\bibitem[Dasgupta et~al.(2011)Dasgupta, Kumar, and Sarl{\'o}s]{DasKumSar11}
A.~Dasgupta, R.~Kumar, and T.~Sarl{\'o}s.
\newblock Fast locality-sensitive hashing.
\newblock In \emph{Proceedings of the 17th ACM SIGKDD international conference
  on Knowledge discovery and data mining}, pages 1073--1081. ACM, 2011.

\bibitem[Davies and Ghahramani(2014)]{DavGha14}
A.~Davies and Z.~Ghahramani.
\newblock The random forest kernel and other kernels for big data from random
  partitions.
\newblock \emph{arXiv preprint arXiv:1402.4293}, 2014.

\bibitem[Fan et~al.(2008)Fan, Chang, Hsieh, Wang, and Lin]{FanChaHsiWanetal08}
R.-E. Fan, J.-W. Chang, C.-J. Hsieh, X.-R. Wang, and C.-J. Lin.
\newblock {LIBLINEAR}: A library for large linear classification.
\newblock \emph{Journal of Machine Learning Research}, 9:\penalty0 1871--1874,
  August 2008.

\bibitem[Fine and Scheinberg(2001)]{FinSch01}
S.~Fine and K.~Scheinberg.
\newblock Efficient {SVM} training using low-rank kernel representations.
\newblock \emph{Journal of Machine Learning Research}, 2:\penalty0 243--264,
  2001.

\bibitem[Frank and Asuncion(2010)]{FraAsu10}
A.~Frank and A.~Asuncion.
\newblock {UCI} machine learning repository, 2010.
\newblock URL \url{http://archive.ics.uci.edu/ml}.

\bibitem[Girosi(1998)]{Girosi98}
F.~Girosi.
\newblock An equivalence between sparse approximation and support vector
  machines.
\newblock \emph{Neural Computation}, 10\penalty0 (6):\penalty0 1455--1480,
  1998.

\bibitem[Girosi and Anzellotti(1993)]{GirAnz93}
F.~Girosi and G.~Anzellotti.
\newblock Rates of convergence for radial basis functions and neural networks.
\newblock In R.~J. Mammone, editor, \emph{Artificial Neural Networks for Speech
  and Vision}, pages 97--113, London, 1993. Chapman and Hall.

\bibitem[Girosi et~al.(1995)Girosi, Jones, and Poggio]{GirJonPog95}
F.~Girosi, M.~Jones, and T.~Poggio.
\newblock Regularization theory and neural networks architectures.
\newblock \emph{Neural Computation}, 7\penalty0 (2):\penalty0 219--269, 1995.

\bibitem[Gray and Moore(2003)]{GraMoo03b}
Alexander~G. Gray and Andrew~W. Moore.
\newblock Rapid evaluation of multiple density models.
\newblock In \emph{Proc.\ Intl.\ Conference on Artificial Intelligence and
  Statistics}, 2003.

\bibitem[Haussler(1999)]{Haussler99}
David Haussler.
\newblock Convolution kernels on discrete structures.
\newblock Technical Report UCS-CRL-99-10, UC Santa Cruz, 1999.

\bibitem[Hochstadt(1961)]{Hochstadt61}
H.~Hochstadt.
\newblock \emph{Special functions of mathematical physics}.
\newblock Dover, 1961.

\bibitem[Huang et~al.(2007)Huang, Guestrin, and Guibas]{HuaGueGui07}
T.~Huang, C.~Guestrin, and L.~Guibas.
\newblock Efficient inference for distributions on permutations.
\newblock In \emph{NIPS}, 2007.

\bibitem[Jin et~al.(2011)Jin, Yang, Mahdavi, Li, and Zhou]{Jinetal11}
R.~Jin, T.~Yang, M.~Mahdavi, Y.F. Li, and Z.H. Zhou.
\newblock Improved bound for the nystrom's method and its application to kernel
  classification, 2011.
\newblock URL \url{http://arxiv.org/abs/1111.2262}.

\bibitem[Kimeldorf and Wahba(1970)]{KimWah70}
G.~S. Kimeldorf and G.~Wahba.
\newblock A correspondence between {B}ayesian estimation on stochastic
  processes and smoothing by splines.
\newblock \emph{Annals of Mathematical Statistics}, 41:\penalty0 495--502,
  1970.

\bibitem[Kondor(2008)]{Kondor08}
R.~Kondor.
\newblock \emph{Group theoretical methods in machine learning}.
\newblock PhD thesis, Columbia University, 2008.
\newblock URL
  \url{http://people.cs.uchicago.edu/~risi/papers/KondorThesis.pdf}.

\bibitem[Kreyszig(1989)]{Kreyszig89}
E.~Kreyszig.
\newblock \emph{Introductory Functional Analysis with Applications}.
\newblock Wiley, 1989.

\bibitem[Krizhevsky(2009)]{Krizhevsky09}
A.~Krizhevsky.
\newblock Learning multiple layers of features from tiny images.
\newblock Technical report, University of Toronto, 2009.

\bibitem[Ledoux(1996)]{Ledoux96}
M.~Ledoux.
\newblock Isoperimetry and gaussian analysis.
\newblock In \emph{Lectures on probability theory and statistics}, pages
  165--294. Springer, 1996.

\bibitem[Lee and Gray(2009)]{LeeGra09}
Dongryeol Lee and Alexander~G. Gray.
\newblock Fast high-dimensional kernel summations using the monte carlo
  multipole method.
\newblock In \emph{Neural Information Processing Systems}. MIT Press, 2009.

\bibitem[MacKay(2003)]{MacKay03}
David J.~C. MacKay.
\newblock \emph{Information Theory, Inference, and Learning Algorithms}.
\newblock Cambridge University Press, 2003.

\bibitem[Matsushima et~al.(2012)Matsushima, Vishwanathan, and
  Smola]{MatVisSmo12}
S.~Matsushima, S.V.N. Vishwanathan, and A.J. Smola.
\newblock Linear support vector machines via dual cached loops.
\newblock In Q.~Yang, D.~Agarwal, and J.~Pei, editors, \emph{The 18th {ACM}
  {SIGKDD} International Conference on Knowledge Discovery and Data Mining,
  {KDD}}, pages 177--185. ACM, 2012.
\newblock URL \url{http://dl.acm.org/citation.cfm?id=2339530}.

\bibitem[Mercer(1909)]{Mercer09}
J.~Mercer.
\newblock Functions of positive and negative type and their connection with the
  theory of integral equations.
\newblock \emph{Philos. Trans. R. Soc. Lond. Ser. A Math. Phys. Eng. Sci.}, A
  209:\penalty0 415--446, 1909.

\bibitem[Micchelli(1986)]{Micchelli86b}
C.~A. Micchelli.
\newblock Interpolation of scattered data: distance matrices and conditionally
  positive definite functions.
\newblock \emph{Constructive Approximation}, 2:\penalty0 11--22, 1986.

\bibitem[Neal(1994)]{Neal94}
R.~Neal.
\newblock Priors for infinite networks.
\newblock Technical Report CRG-TR-94-1, Dept.~of Computer Science, University
  of Toronto, 1994.

\bibitem[Rahimi and Recht(2008)]{RahRec08}
A.~Rahimi and B.~Recht.
\newblock Random features for large-scale kernel machines.
\newblock In J.C. Platt, D.~Koller, Y.~Singer, and S.~Roweis, editors,
  \emph{Advances in Neural Information Processing Systems 20}. MIT Press,
  Cambridge, MA, 2008.

\bibitem[Rahimi and Recht(2009)]{RahRec09}
Ali Rahimi and Benjamin Recht.
\newblock Weighted sums of random kitchen sinks: Replacing minimization with
  randomization in learning.
\newblock In \emph{Neural Information Processing Systems}, 2009.

\bibitem[Rasmussen and Williams(2006)]{RasWil06}
C.~E. Rasmussen and C.~K.~I. Williams.
\newblock \emph{Gaussian Processes for Machine Learning}.
\newblock MIT Press, Cambridge, MA, 2006.

\bibitem[Ratliff et~al.(2007)Ratliff, Bagnell, and Zinkevich]{RatBagZin07}
N.~Ratliff, J.~Bagnell, and M.~Zinkevich.
\newblock (online) subgradient methods for structured prediction.
\newblock In \emph{Eleventh International Conference on Artificial Intelligence
  and Statistics (AIStats)}, March 2007.

\bibitem[{Sch\"olkopf} et~al.(1998){Sch\"olkopf}, Smola, and
  {M\"uller}]{SchSmoMul98}
B.~{Sch\"olkopf}, A.~J. Smola, and K.-R. {M\"uller}.
\newblock Nonlinear component analysis as a kernel eigenvalue problem.
\newblock \emph{Neural Comput.}, 10:\penalty0 1299--1319, 1998.

\bibitem[Sch{\"o}lkopf and Smola(2002)]{SchSmo02}
Bernhard Sch{\"o}lkopf and A.~J. Smola.
\newblock \emph{Learning with Kernels}.
\newblock {MIT} Press, Cambridge, MA, 2002.

\bibitem[Smola(1998)]{Smola98}
A.~J. Smola.
\newblock \emph{Learning with Kernels}.
\newblock PhD thesis, Technische Universit\"at Berlin, 1998.
\newblock GMD Research Series No.~25.

\bibitem[Smola and Sch\"olkopf(2000)]{SmoSch00}
A.~J. Smola and B.~Sch\"olkopf.
\newblock Sparse greedy matrix approximation for machine learning.
\newblock In \emph{Proceedings of the International Conference on Machine
  Learning}, pages 911--918, San Francisco, 2000. Morgan Kaufmann Publishers.

\bibitem[Smola et~al.(1998{\natexlab{a}})Smola, {Sch\"olkopf}, and
  {M\"uller}]{SmoSchMul98}
A.~J. Smola, B.~{Sch\"olkopf}, and K.-R. {M\"uller}.
\newblock General cost functions for support vector regression.
\newblock In T.~Downs, M.~Frean, and M.~Gallagher, editors, \emph{Proc.\ of the
  Ninth Australian Conf.\ on Neural Networks}, pages 79--83, Brisbane,
  Australia, 1998{\natexlab{a}}. University of Queensland.

\bibitem[Smola et~al.(1998{\natexlab{b}})Smola, Sch{\"o}lkopf, and
  M{\"u}ller]{SmoSchMul98b}
A.~J. Smola, B.~Sch{\"o}lkopf, and K.-R. M{\"u}ller.
\newblock The connection between regularization operators and support vector
  kernels.
\newblock \emph{Neural Networks}, 11\penalty0 (5):\penalty0 637--649,
  1998{\natexlab{b}}.

\bibitem[Smola et~al.(2001)Smola, {\'Ov\'ari}, and Williamson]{SmoOvaWil01}
A.~J. Smola, Z.~L. {\'Ov\'ari}, and R.~C. Williamson.
\newblock Regularization with dot-product kernels.
\newblock In T.~K. Leen, T.~G. Dietterich, and V.~Tresp, editors,
  \emph{Advances in Neural Information Processing Systems 13}, pages 308--314.
  {MIT} Press, 2001.

\bibitem[Steinwart and Christmann(2008)]{SteChr08}
Ingo Steinwart and Andreas Christmann.
\newblock \emph{Support Vector Machines}.
\newblock Information Science and Statistics. Springer, 2008.

\bibitem[Taskar et~al.(2004)Taskar, Guestrin, and Koller]{TasGueKol04}
B.~Taskar, C.~Guestrin, and D.~Koller.
\newblock Max-margin {M}arkov networks.
\newblock In S.~Thrun, L.~Saul, and B.~Sch\"{o}lkopf, editors, \emph{Advances
  in Neural Information Processing Systems 16}, pages 25--32, Cambridge, MA,
  2004. MIT Press.

\bibitem[Teo et~al.(2010)Teo, Vishwanthan, Smola, and Le]{TeoVisSmoLe10}
Choon~Hui Teo, S.~V.~N. Vishwanthan, A.~J. Smola, and Quoc~V. Le.
\newblock Bundle methods for regularized risk minimization.
\newblock \emph{Journal of Machine Learning Research}, 11:\penalty0 311--365,
  January 2010.

\bibitem[Tropp(2010)]{Tropp10}
J.~A. Tropp.
\newblock Improved analysis of the subsampled randomized hadamard transform.
\newblock \emph{CoRR}, abs/1011.1595, 2010.
\newblock URL \url{http://arxiv.org/abs/1011.1595}.

\bibitem[Tsuda et~al.(2002)Tsuda, Kin, and Asai]{TsuKinAsa02}
K.~Tsuda, T.~Kin, and K.~Asai.
\newblock Marginalized kernels for biological sequences.
\newblock \emph{Bioinformatics}, 18 (Suppl.~2):\penalty0 S268--S275, 2002.

\bibitem[Vapnik et~al.(1997)Vapnik, Golowich, and Smola]{VapGolSmo97}
V.~Vapnik, S.~Golowich, and A.~J. Smola.
\newblock Support vector method for function approximation, regression
  estimation, and signal processing.
\newblock In M.~C. Mozer, M.~I. Jordan, and T.~Petsche, editors, \emph{Advances
  in Neural Information Processing Systems 9}, pages 281--287, Cambridge, MA,
  1997. {MIT} Press.

\bibitem[Wahba(1990)]{Wahba90}
G.~Wahba.
\newblock \emph{Spline Models for Observational Data}, volume~59 of
  \emph{CBMS-NSF Regional Conference Series in Applied Mathematics}.
\newblock {SIAM}, Philadelphia, 1990.

\bibitem[Williams(1998)]{Williams98}
C.~K.~I. Williams.
\newblock Prediction with {G}aussian processes: From linear regression to
  linear prediction and beyond.
\newblock In M.~I. Jordan, editor, \emph{Learning and Inference in Graphical
  Models}, pages 599--621. Kluwer Academic, 1998.

\bibitem[Williams and Seeger(2001)]{WilSee01}
Christoper K.~I. Williams and Matthias Seeger.
\newblock Using the {N}ystrom method to speed up kernel machines.
\newblock In T.~K. Leen, T.~G. Dietterich, and V.~Tresp, editors,
  \emph{Advances in Neural Information Processing Systems 13}, pages 682--688,
  Cambridge, MA, 2001. {MIT} Press.

\bibitem[Williamson et~al.(2001)Williamson, Smola, and
  {Sch\"olkopf}]{WilSmoSch01}
R.~C. Williamson, A.~J. Smola, and B.~{Sch\"olkopf}.
\newblock Generalization bounds for regularization networks and support vector
  machines via entropy numbers of compact operators.
\newblock \emph{IEEE Trans. Inform. Theory}, 47\penalty0 (6):\penalty0
  2516--2532, 2001.

\end{thebibliography}

\end{document}